%% file: DeepCAST_paper.tex
\documentclass[lettersize,journal]{article}

\usepackage{arxiv}

\usepackage{amssymb,amsmath,amsthm,amsthm}
\allowdisplaybreaks 
\usepackage{amsbsy}
\usepackage{mathtools}
\usepackage{enumitem}
\usepackage{algorithm}
\usepackage{array}
\usepackage{textcomp}
\usepackage{stfloats}
\usepackage{url}
\usepackage{verbatim}
\usepackage{graphicx}
\usepackage{cite}
\usepackage{multirow}
\usepackage{tikz}
\usepackage{soul}
\usepackage[hypertexnames=false,hidelinks]{hyperref}
\usepackage{subcaption}
\usepackage{caption}

\newtheorem{theorem}{Theorem}
\newtheoremstyle{nonitalic} 
{3pt} 
{3pt} 
{\upshape} 
{} 
{\bfseries} 
{.} 
{.5em} 
{} 

\theoremstyle{nonitalic}
\theoremstyle{nonitalic}
\newtheorem{remark}[theorem]{Remark}

\theoremstyle{plain} 
\newtheorem{definition}[theorem]{Definition}

\newtheorem{lemma}[theorem]{Lemma}

\usepackage{booktabs}

\newcommand{\N}{\mathbb{N}}
\newcommand{\R}{\mathbb{R}}

\newcommand{\argmin}{\operatorname{argmin}}

\newcommand{\bx}{{\boldsymbol{x}}}

\newcommand{\lowhigh}{\Psi_{\textnormal{L2H}}}
\newcommand{\lowhighp}{\Psi_{\textnormal{PL2H}}}
\newcommand{\highlow}{\Psi_{\textnormal{H2L}}}
\newcommand{\highlowp}{\Psi_{\textnormal{PH2L}}}
\newcommand{\highlowk}{\Psi_{\textnormal{H2L},k}}
\newcommand{\highlowalpha}{\Psi_{\textnormal{H2L},\alpha}}

\newcommand{\xHsim}{\bx^\textnormal{sim}_{\textnormal{H}}}
\newcommand{\xHsimk}{\bx^\textnormal{sim}_{\textnormal{H},k}}
\newcommand{\xHsimki}{\bx^\textnormal{sim}_{\textnormal{H},k,i}}
\newcommand{\xL}{\bx_{\textnormal{L}}}
\newcommand{\xLsim}{\bx^\textnormal{sim}_{\textnormal{L}}}
\newcommand{\xH}{\bx_{\textnormal{H}}}
\newcommand{\xP}{\bx_{\textnormal{P}}}

\usepackage{xcolor}
\definecolor{azure}{rgb}{0.0, 0.5, 1.0}
\definecolor{darkred}{RGB}{175, 0, 0}
\definecolor{niceblue}{RGB}{0, 102, 204} 
\definecolor{nicegreen}{RGB}{76, 175, 80} 
\definecolor{LightGrey}{RGB}{211, 211, 211} 
\definecolor{darkred}{rgb}{0.8, 0.0, 0.0}

\usepackage{tcolorbox}
\newtcolorbox{low-dose-CA-problem}{
	colback=azure!5!white,
	colframe=azure!75!black,
	title=The Low-to-High Contrast Agent  Problem
}

\usepackage{tikz}
\usepackage{tikz-network}
\usetikzlibrary{3d,arrows.meta}

\date{}

\begin{document}
\title{LIP-CAR: contrast agent reduction by a deep learned inverse problem}

\author{Davide Bianchi \\
School of Mathematics (Zhuhai), \\ Sun Yat-sen University,\\ Zhuhai,  519082, China. \\ \texttt{bianchid@sysu.edu.cn}. \\
\And Sonia Colombo Serra \\
Centro Ricerche Bracco,\\ Bracco Imaging SpA,\\ Colleretto Giacosa, 10010, Italy. \\ \texttt{sonia.colombo@bracco.com}. \\
\And Davide Evangelista \\
Department of Computer Science and Engineering, \\ University of Bologna,\\ Bologna, 40126, Italy. \\
\texttt{davide.evangelista5@unibo.it} \\
\And Pengpeng Luo \\
Bracco Imaging Medical Technologies Co. Ltd, \\ Shanghai, 200000, China. \\ \texttt{ethe.luo@bracco.com}. \\
\And Elena Morotti \\
Department of Political and Social Sciences, \\ University of Bologna,\\ Bologna, 40126, Italy. \\ \texttt{elena.morotti4@unibo.it}. \\
\And Giovanni Valbusa \\
Global R\&D,\\ Bracco Imaging SpA,\\ Colleretto Giacosa, 10010, Italy. \\ \texttt{giovanni.valbusa@bracco.com}.}

\maketitle

\begin{abstract}
The adoption of contrast agents in medical imaging protocols is crucial for accurate and timely diagnosis. While highly effective and characterized by an excellent safety profile, the use of contrast agents has its limitation, including rare risk of allergic reactions, potential environmental impact and economic burdens on patients and healthcare systems. In this work, we address the contrast agent reduction (CAR) problem, which involves reducing the administered dosage of contrast agent while preserving the visual enhancement. 
The current literature on the CAR task is based on deep learning techniques within a fully image processing framework. These techniques digitally simulate high-dose images from images acquired with a low dose of contrast agent. We investigate the feasibility of a ``learned inverse problem'' (LIP) approach, as opposed to the end-to-end paradigm in the state-of-the-art literature. 

Specifically, we learn the image-to-image operator that maps high-dose images to their corresponding low-dose counterparts, and we frame the CAR task as an inverse problem. We then solve this problem through a regularized optimization reformulation. Regularization methods are well-established mathematical techniques that offer robustness and explainability. Our approach combines these rigorous techniques with cutting-edge deep learning tools. Numerical experiments performed on pre-clinical medical images confirm the effectiveness of this strategy, showing improved stability and accuracy in the simulated high-dose images. 
 
\end{abstract}


\section{Introduction}\label{sec:intro}
\input{SECTION_introduction}

\section{State of the art}\label{sec:state_of_art}
\input{SECTION_state_of_art}

\section{The proposed method}\label{sec:model}
\input{SECTION_model_setting}


\section{Experimental setup}\label{sec:exp_setup}
\input{SECTION_experimental_setup}

\section{Numerical experiments}\label{sec:num_results}
\input{SECTION_numerical_experiments}

\section{Conclusions}\label{sec:conclusions}
This paper introduces a novel application of neural networks, showing great potential. While many studies have already exploited deep architectures to extract imaging patterns and used them to solve imaging problems with end-to-end approaches directly, we use a convolutional network to learn an operator map in the inverse direction. We then tackle the imaging task as an inverse problem. This approach allows us to handle imaging applications with mathematically grounded tools that belong to the well-established class of optimization and regularization techniques.

As an explanatory application of our approach, we consider a  medical challenge involving contrast-agent imaging. The problem, named CAR (Contrast Agent Reduction), consists of digitally simulating the image achieved by injecting the patients with a high dosage of CA medium  while subjecting them to a reduced dose (resulting in images with lower contrast and reduced lesion detectability). 
Despite fast growing interest in the scientific community, the literature on the CAR  task from a purely image-to-image approach is still limited. 
Thus, we have poured our main effort into designing a framework (LIP-CAR) that could be effectively necessary and used in practice.
Additionally, as our LIP-CAR proposal solves a regularization model,
it does not turn the image processing into a fully unexplainable black box, which is an important feature for medical applications.

We have assessed the LIP-CAR numerical consistency with our mathematical, in-depth theoretical analysis through several experiments on a real, original pre-clinical dataset.  
Specifically, we have analyzed the predictive accuracy of the simulated images  and assessed the LIP-CAR robustness.
The quality of the simulated images produced by the LIP-CAR method improves upon or is comparable to those obtained by the state-of-the-art direct approach. Moreover, while direct methods are very sensitive to perturbations, LIP-CAR is stable. 
The overall LIP-CAR proposal is thus a flexible and well-performing tool for the CAR problem. 

To conclude, we highlight that this work represents an introductory study of the LIP-CAR approach. Here, we have not focused on many aspects where it can be further improved. To set some examples, we have not fine-tuned the choice of the optimization formulation in terms of the data-fitting term and, above all, of the regularizer. Even the fine-tuning of the regularization parameter(s) plays an important role in the final results, but this study is out of the scope of this paper. 

\section*{Acknowledgements}
The results presented in this study stem from the  DeepCAST project, funded by Bracco Imaging Medical Technologies Co. Ltd.\\
D. Bianchi is partially supported by the Startup Fund of Sun~Yat-sen~University. \\
D. Evangelista and E. Morotti are partially supported by “Gruppo Nazionale per il Calcolo Scientifico (GNCS-INdAM)” (Progetto 2024 “Deep Variational Learning: un approccio combinato per la ricostruzione di immagini"), and the PRIN 2022 project “STILE: Sustainable Tomographic Imaging with Learning and rEgularization”, project code: 20225STXSB, funded by the European Commission under the NextGeneration EU programme.\\
The authors thank A. Fringuello Mingo and F. La Cava for conducting the MRI preclinical study on a rodent tumour model.

\bibliographystyle{abbrv}
\bibliography{DeepCAST_paper}
\end{document}

%% file: SECTION_introduction.tex

\noindent Modern technologies used for Computed Tomography (CT) and Magnetic Resonance Imaging (MRI) offer an unquestionable value to health care, but the visibility of very small or low-contrasted structures is not always guaranteed. 
Contrast Agents (CAs) are chemical compounds introduced into the patient's body to alter the way the body interacts with physical stimuli (X-rays or magnetic fields) and, hence, to better distinguish (or ``contrast'') lesions from healthy surrounding tissues. 
CAs are used in various diagnostic tasks, like searching for tumors~\cite{claussen1985application,zhou2013gadolinium}, mesenteric ischemia~\cite{menke2010diagnostic}, pulmonary artery embolism~\cite{nicod1987pulmonary}, to cite just a few applications. 
See Figure~\ref{fig:pre-low-high}.

\begin{center}
	\begin{figure}
		\centering%
		\begin{minipage}{0.22\textwidth}
			\centering
			\includegraphics[width=\textwidth]{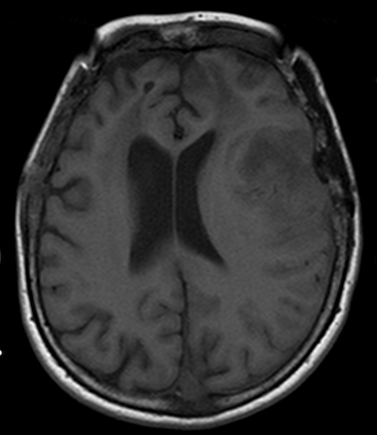}\ 
		\end{minipage}\hspace{1cm}
		\begin{minipage}{0.22\textwidth}
			\centering
			\includegraphics[width=\textwidth]{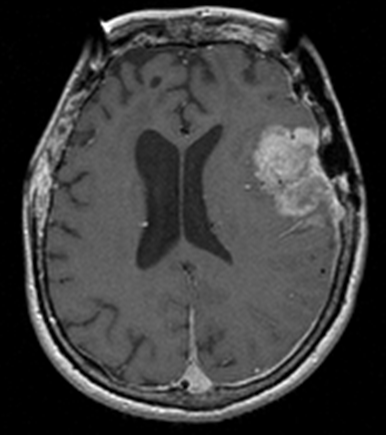} 
		\end{minipage}\caption{ MRI images of a human brain\cite{vassantachart2023segmentation}, acquired without CA (on the left) and with a high dosage of CA (on the right).   
       }\label{fig:pre-low-high}
	\end{figure}
\end{center}

Iodine- and gadolinium-based CAs are commonly employed in CT and MRI scans, respectively. These substances are generally deemed safe when administered within prescribed dose thresholds. Nevertheless, achieving a balance between administering a minimal amount of CA and obtaining high-fidelity imaging poses a notable challenge. For instance, despite rare or very rare, there are documented adverse effects associated with the use of these CAs \cite{bush1991acute,singh2008iodinated}, and regulatory bodies like the European Medicines Agency have imposed restrictions on the use of certain linear gadolinium based CAs as a precaution about their retention/deposition in the brain and other tissues \cite{EMA,runge2017critical}, even if in the absence of any associated clinical sign. 


In this scenario the scientific community is looking for solutions to reduce the use of CAs. In the following, we refer to the Contrast Agent Reduction (CAR) problem as the task of minimizing the level of CA while preserving the image enhancement granted by a high-level dose of CA.

Thanks to the recent advancements in deep learning, new purely digital imaging techniques have been developed to address the CAR problem. See for example \cite{gong2018deep, bone2021contrast, haubold2021contrast, luo2021deep, montalt2021reducing, pasumarthi2021generic, bone2022dose, ammari2022can, wang2022deep, haase2023reduction,muller2023using}.
In these works, deep neural network (NN) operators are used, taking an image (or a stack of images) acquired with a low dosage of CA and directly computing a simulation of the same image acquired with a high CA dosage.
These approaches are learned ``end-to-end'' processing methods. In a compact way, we can write them as: 
\begin{equation*}
    \lowhigh \colon \xL \longmapsto \xH,
\end{equation*}
where 
$\lowhigh$ is a pre-trained NN mapping low-dose to high-dose images, $\xL$ is the acquired low-dose image and $\xH$ is the desired high-dose image.

However, when dealing with real-world applications, every acquired image is contaminated by some unavoidable error of typically unknown distribution and sources. We take that into account by considering:
\begin{equation}\label{eq:noiseIn}
    \xL^\delta \coloneqq \xL + \boldsymbol{\eta}^{in},
\end{equation}
where $\xL$ represents the clean, inaccessible image and $\boldsymbol{\eta}^{in}$ is the intrinsic perturbation due to the (CT or MRI processing) imaging system.
We reasonably assume $\|\boldsymbol{\eta}^{in} \|_2\leq \delta$ with $\delta>~0$ and call $\delta$ the noise intensity.
Under this notation, the learned end-to-end method takes the following final form:
\begin{equation}\label{eq:EndToEnd}
\mbox{Given}\;\; \xL^\delta, \quad \mbox{compute}\;\; \xHsim \coloneqq \lowhigh (\xL^\delta).
\end{equation}

However, despite NNs achieving outstanding performance in end-to-end schemes for various imaging tasks, accuracy typically comes at the expense of robustness in deep learning. 
State-of-the-art literature has demonstrated that optimal stability and accuracy cannot be achieved at the same time \cite{antun2020instabilities,gottschling2020troublesome,colbrook2021can}. For instance, the network performance in end-to-end frameworks can be drastically affected by the presence of unquantifiable noise on the images when only a few noise intensities have been considered during the training phase~\cite{goodfellow2014explaining}. This can cause artifacts to appear on the output image~\cite{morotti2021green, evangelista2023ambiguity}, and even false positive details may arise~\cite{gottschling2020troublesome}. Given the computational impracticality of training a NN across all potential noise intensities and distributions, or any generic perturbation, this presents a major challenge.\\
Additionally, the instability of deep-learning-based results is a significant concern in the medical field, which is one of the main areas involved in the Explainable-AI revolution~\cite{Tjoa2021ExplainableAI}. This new discipline rejects the black-box nature of deep learning which prevents interpretation and thus undermines the reliability of medical practices. Instead, it promotes initiatives to advance data-based, mathematically- and technically-grounded medical applications of NNs. See~\cite{ruthotto2020deep,cheng2023continuous,tai2024pottsmgnet,haber2017stable,ruiz2023neural} for some recent examples in this direction. \\

Due to all the aforementioned issues, we propose a paradigm shift. Instead of using a direct end-to-end approach, we look at the CAR problem as an inverse problem. 
We name our proposal LIP-CAR, referring to the Learned Inverse Problem for Contrast Agent Reduction, as illustrated in Figure~\ref{fig:learned_inverse_problem}. LIP-CAR operates in two primary stages.


Firstly, we train an image-to-image NN, denoted as $\highlow$, which is designed to map the acquired high-dose images to their (simulated) low-dose counterparts:
\begin{equation}\label{eq:CAROperator}
	\highlow: \xH \longmapsto \xL.
\end{equation}
This NN operates inversely to $\lowhigh$, as illustrated in Figure~\ref{fig:learned_inverse_problem}.
We denote the mapping function $\highlow$ as the ``LIP forward operator'', as it represents the forward operation of our inverse problem, and it simulates the distribution of concentration of the chemical compound through the body by mapping high-dose to low-dose images.

Secondly, taking into account \eqref{eq:noiseIn}, we address the associated inverse problem:
\begin{equation}\label{eq:LIP}
\mbox{Given}\;\; \xL^\delta, \quad \mbox{solve}\;\; \highlow (\bx) = \xL^\delta,
\end{equation}
by means of regularized optimization techniques. \\ 
This shift in approach is crucial as it allows the use of a wide range of well-established and robust techniques from inverse problems regularization theory. On the one hand, we can stabilize the simulated high-dose images against perturbations in the low-dose acquired images; on the other hand, we can concentrate on enhancing the accuracy of the NNs with the confidence that, in the limit case, the simulated high-dose images converge to the true high-dose images.
\begin{figure*}[hbt!]
    \centering
    \includegraphics[width=0.9\textwidth]{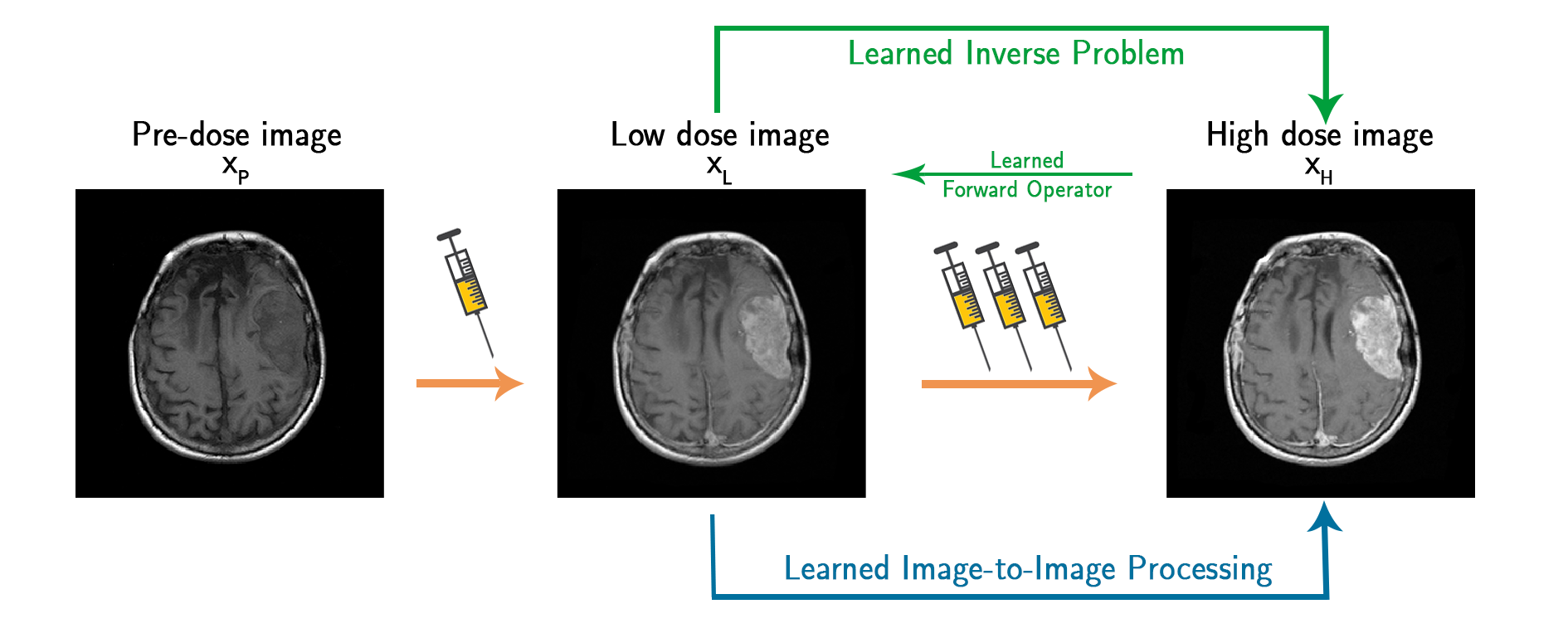}
    \caption{ Visual representation of the paradigm change for the CAR task: from an end-to-end NN-based approach to a Learned Inverse Problem framework.}
 \label{fig:learned_inverse_problem}
\end{figure*}

The contribution of this paper can thus be summarized as follows.
\begin{enumerate}
    \item We introduce a new approach, LIP-CAR, for the CAR task that combines mathematical regularization techniques with cutting-edge deep learning tools in an accurate, well-performing two-step procedure. The convergence and stability of the method are guaranteed by a rigorous theoretical foundation.
    \item LIP-CAR represents a versatile framework, realized here through distinct implementations. At the stage of formulating the inverse problem statement, we test two different regularizers to discuss the potential and variety of regularization priors. Furthermore. the adaptability of the LIP forward operator easily allows for embedding registered pre-dose injection images, a known enhancer of data-driven performance. 
    \item Our numerical experiments, performed on a data set of  pre-clinical images, assess that LIP-CAR can achieve high accuracy and is more robust than the end-to-end approach.
\end{enumerate}

The manuscript is organized as follows: In Section~\ref{sec:state_of_art} we briefly revise the state-of-the-art literature about the CAR problem. In Section~\ref{sec:model} we introduce our new model method LIP-CAR along with a mathematical analysis of its regularization and stability properties. Sections~\ref{sec:exp_setup}-\ref{sec:num_results} present the experimental setup and  numerical results, respectively, where we compare the end-to-end method \eqref{eq:EndToEnd} and our regularized proposal \eqref{eq:LIPCAR}. Finally, in Section~\ref{sec:conclusions} we draw the conclusions and outline future directions of research.

%% file: SECTION_state_of_art.tex

Mathematical models for the CA distribution-excretion phenomenon have been introduced since the 1990s. We refer the reader to the seminal paper by Tofts et al. \cite{tofts1999estimating} and the references therein.
However, on the one hand, the exact mechanism still remains unknown, and this is primarily due to the difficulty in fully understanding and accurately modeling the extensive range of complex physical and chemical interactions that occur within each different body following the administration of CAs. On the second hand, numerically solving the partial differential equations describing the phenomenon is challenging due to the ill-posed nature of the problem. 

With the rise of deep learning, interest in the CAR problem as we intend it here and, more generally, in CA-related biomedical imaging problems, has seen a new surge in the scientific community. NNs are universal approximators \cite{hornik1989multilayer}, suggesting that with the correct architecture and training, the unknown CA distribution phenomenon can be well-approximated. In that sense recent results are promising. For example, in \cite{sainz2024exploring}, the authors use Physics-Informed Neural Networks in combination with pharmacokinetic models to retrieve vascularization parameters in dynamic contrast-enhanced MRI. In \cite{della2020deepspio},  a novel neural network architecture is presented to quantify the concentration of superparamagnetic iron oxide particles, which are used as CAs in several MRI diagnostic tasks. In \cite{yang2021robust}, an unsupervised method using the deep image prior technique is introduced for phase unwrapping in quantitative phase imaging, which is used for CA-free biomedical imaging. In \cite{bone2021contrast}, the authors train a V-Net to simulate high-dose MRI brain images from low-dose images acquired using only 25\% of the standard gadolinium-based CA~dose.

For the CAR problem, the state-of-the-art literature now strongly relies on NNs, and it always implements the direct end-to-end approach \eqref{eq:EndToEnd} by learning a $\lowhigh$ operator. Notably, in papers such as \cite{bone2021contrast,bone2022dose,pasumarthi2021generic,gong2018deep,haase2023reduction}, procedures with pre-dose image acquisition are considered and the authors have boosted the end-to-end operator defining it as:
\begin{equation}\label{eq:EndToEnd_withP}
    \Psi_{\textnormal{PL2H}}: (\xP, \xL) \longmapsto \xH.
\end{equation}
In this case, both the pre-dose $\xP$ and low-dose images are passed as input to the NN to predict the high-dose one. This approach widely enhances the end-to-end performance, above all in terms of image accuracy.

However, one of the inherent challenges of the direct low-to-high approach lies within the task itself. We are transitioning directly from an image with smoother characteristics (low-dose) to one that exhibits more pronounced discontinuities (high-dose). As widely studied in the literature, this image enhancement process is not robust and prone to generate errors in case of even small inconsistencies on the data \cite{hansen1998rank, scherzer2009variational}. 
Consequently, the end-to-end approach for CAR application may inherit ill-posedness from the underlying imaging task and its robustness with respect to perturbations (both in terms of extra noise and generalization error to unseen data) must be assessed. See, for example, \cite{muller2024diffusion}, where false positives (artificial lesions) may appear in the digitally simulated high-dose images. 

Interestingly, in many applications different than CA but still involving the digital processing of signals and images, the end-to-end usage of NNs is giving way to new hybrid approaches, where deep learning tools are embedded into variational schemes in a mathematically grounded scenario \cite{LOPEZTAPIA2021_PastPresentFuture, Genzel_2022_RobustnessIncluded}. 
In the vast available literature, in fact, deep networks have already been successfully used to learn optimal algorithmic parameters \cite{arridge2019solving, Monga_2021_unrolling}, to define suitable regularizers \cite{schwab2019deep,li2020nett,bianchi2023uniformly,bianchi2023graph,bianchi2023data,cascarano2022plug}, or to speed up the reconstruction procedure \cite{evangelista2023rising}. All these studies have inspired us to design a novel use of  NNs for the CAR problem.

%% file: SECTION_model_setting.tex
%

We begin this section by establishing some of the notations and assumptions used throughout the remainder of the manuscript.
We work on 2D slices with dimensions $H\times W$, although the framework is also applicable to 3D imaging with few adjustments.
All images $\bx$ are grayscale and normalized to $[0,1]$, for simplicity. They can be represented as functions $\bx \colon P \to [0,1]$, where $P=\{i \mid i=1,\ldots,n\}$ is the ordered set of pixels, and $n=H\cdot W$ is the fixed total number of pixels. Thus, we can associate the space of images with the hypercube $[0,1]^n \coloneqq [0,1]\times\cdots \times [0,1]$. In other words, for us, an image $\bx$ is an element of the hypercube $[0,1]^n$, 
$$
\bx = (\bx(1), \ldots, \bx(n)) \in [0,1]^n\subset \R^n.
$$
We use the standard notation $\|\cdot\|_p$  for the $\ell^p$-norms, $p\in [1,\infty]$.  That is, $\|\bx\|_p \coloneqq (\sum_{i=1}^n |\bx(i)|^p)^{1/p}$ for $p<\infty$ and 
$\|\bx\|_\infty \coloneqq \max_{i=1,\ldots,n} |\bx(i)|$ for $p=\infty$.

As previously introduced, we denote as $\xP$, $\xL$, and $\xH$ the pre-dose, low-dose, and high-dose images, respectively.  The term ``high-dose" denotes a standard reference level established either by a regulatory body or by laboratory protocols. Conversely, ``low-dose" refers to instances where a reduced quantity of CA is administered to the patient, while ``pre-dose" indicates the image acquired prior to any CA injections. These images can be acquired with a tomographic procedure or the MRI technique and then reconstructed as gray-scale images. We denote the digitally-simulated high-dose image as $\xHsim$, representing the target of the CAR imaging problem.

We can reasonably assume that the space of our images generated for a specific medical application belongs to a subset \(\mathcal{X}\) of the space of grayscale images, that is, \(\mathcal{X} \subset [0,1]^n\). 

At last, we assume that there exists a natural phenomenon \(\mathcal{F} \colon \mathcal{X} \to \mathcal{X}\) that decreases the CA concentration. However, this phenomenon is unknown or impossible to model precisely due to the large number of variables and hidden mechanisms involved.

\subsection{NN operators for learning an inverse problem}\label{ssec:NN}

From a mathematical point of view, a neural network ~$\Psi$ is a chain of $D$ compositions of affine transformations and possibly non-linear activation functions, and it can be compactly defined as:
\begin{equation*}
    \Psi(\bx) \coloneqq \sigma_D\biggl( W_D \Bigl(\dots \bigl(\sigma_1( W_1\bx + \boldsymbol{b}_1 ) \bigr) \dots  \Bigr) + \boldsymbol{b}_D \biggr)
\end{equation*}
where  $W_k$ and $\boldsymbol{b}_k$ are matrices and bias vectors, respectively,  defining the set of  parameters $\{W_{k}, \boldsymbol{b}_{k}\}$,  and $\sigma_k$ are possibly non-linear functions. All modern architectures applied in image processing tasks generate continuous NN operators. 

The overall combination of $W_k, \boldsymbol{b}_k$ and $\sigma_k$, makes the architecture of a NN.
The gist of a NN is that a subset $\Theta = \{W_{k'}, \boldsymbol{b}_{k'}\}$ of the parameter set is free and \emph{trained} by minimizing a loss function $\mathcal{L}$ over a collection of N input-output pairs $(\bx^i_1, \boldsymbol{x}^i_2)_i$ for $i = 1, \dots N$, that is:
\begin{equation}\label{eq:genericLoss}
\Theta_t = \underset{\Theta}{\argmin}\left\{\sum_{i=1}^N \mathcal{L}(\Psi(\bx_1^i), \bx^i_2)\right\}.
\end{equation}
Once the training phase is completed, $\Theta_t$ remains fixed, and so does $\Psi$.  For a basic overview of NNs we refer to \cite{bishop2024deep}.

The widely used $\lowhigh$ operators are trained using couples $(\xL^i, \xH^i)_i$ to learn the (pseudo) inverse $\mathcal{F}^\dagger$ function.


Still exploiting NN operators, our method introduces a change of perspective from the paradigm in the state-of-the-art literature, as we aim to learn the $\mathcal{F}$ phenomenon. 

In fact, our first task is to find the operator $\highlow\colon \mathcal{X} \to~\R^n_+$ introduced in Equation \eqref{eq:CAROperator}, transitioning from a signal with higher discontinuity (high-dose image) to a smoother one (low-dose image).
It is expected to be more well-behaved and stable than $\lowhigh$.\\
To obtain it, we train $\highlow$ using the available high-dose data as input $\bx^i_1$ images and the corresponding low-dose images as target $\bx^i_2$, in the loss function \eqref{eq:genericLoss}.
The trained operator thus generates a simulated low-dose image $\xLsim$, given the high-dose one: 
\begin{equation*}
    \xLsim \coloneqq \highlow(\xH).
\end{equation*}
The $\highlow$ network, mimicking $\mathcal{F}$, represents the LIP forward operator. 

The presence of the forward operator allows the CAR reformulation as the inverse problem stated in \eqref{eq:LIP}. 
Indeed, our second task is the computation of the high-dose simulated image $\xHsim$, defined as a solution of the $\highlow$-associated inverse problem, i.e.:
\begin{equation*}
    \xHsim \quad \mbox{such that} \quad \highlow(\xHsim) = \xL.
\end{equation*}

We remark that we also consider the LIP forward neural network
$\Psi_{\textnormal{PH2L}}$, defined as:
\begin{equation}\label{eq:CAROperator_withP}
    \Psi_{\textnormal{PH2L}}: (\xP, \xH) \longmapsto \xL,
\end{equation}
inspired by the operator in \eqref{eq:EndToEnd_withP}.
It can be used in place of $\highlow$ in Equation \eqref{eq:LIP}, when the no CA-enhanced signal $\xP$ is available. 
For the sake of clarity, we will henceforth omit explicit consideration of this case in the remainder of this section, as the associated modifications to the LIP-CAR framework are trivial in nature from the baseline model using the \(\highlow\) operator.

\subsection{The regularized LIP-CAR model}\label{ssec:regLIP}

It is known in the literature that solving an inverse problem through an optimization reformulation can offer advantages in terms of efficiency, flexibility, and robustness compared to direct inversion methods.
Formally, we substitute the original operator $\highlow$ with a family $\{\highlowalpha\}_{\alpha\in \R{+}}$ of stable operators $\highlowalpha \colon \mathcal{X} \to \mathcal{Y} \supseteq \mathcal{X}$  such that point-wise converge to the (pseudo) inverse of $\highlow$, as $\alpha$ goes to zero. See for example~\cite{Engl1996}.

We thus reformulate the LIP  problem~\eqref{eq:LIP} with the regularized optimization problem:
\begin{equation}\label{eq:LIPCAR}
\xHsim \in \underset{\bx \in \mathcal{X}}{\argmin} \left\{ \frac{1}{2} \|\highlow(\bx) - \xL^\delta \|_2^2 + \alpha \mathcal{R}(\bx) \right\}
\end{equation}
which defines our LIP-CAR model.
In this expression, the least square term forces the data fidelity of the simulated high-dose image, considering the presence of white intrinsic noise $\boldsymbol{\eta}^{in}$ (described in Equation \eqref{eq:noiseIn}).
The component $\mathcal{R}(\bx)$ serves as a regularization term, and the regularization parameter $\alpha >0$ balances the trade-off between data fidelity and the regularization effect. The regularization operator $\mathcal{R}$ is of crucial importance. In the limit (and unrealistic case), the perfect regularizer is given by:
\begin{equation*}
    \mathcal{R}(\bx) \coloneqq \begin{cases}
    0 & \mbox{if } \bx = \xH,\\
    +\infty & \mbox{otherwise}.
    \end{cases}
\end{equation*}
This suggests that, if we have access to a-priori information on the solution $\xH$, we can tailor a specific $\mathcal{R}$ to incorporate such information and penalize all the candidate solutions $\bx$ that do not present the features we aim to recover.

To set $\mathcal{R}$, we exploit the well-established regularization techniques that have already been extensively studied to offer a reliable approximation of a solution to the LIP inverse problem~\eqref{eq:LIP}. This grants robustness and explainability to our approach, as we show in the Section~\ref{ssec:conv}.

For the sake of simplicity and to keep this paper self-contained, in the numerical experiments of Section~\ref{sec:num_results}, we focus on a generalized Tikhonov $\ell^2-\mathcal{R}$ formulation. 
We consider the $\ell_1$-Total Variation (TV) norm:
\begin{equation}\label{eq:TV}
    \mathcal{R}(\bx) = \|\nabla \bx\|_1
\end{equation}
where $\nabla$ denotes the gradient operator. The TV prior encourages sparsity in image gradients, which leads to sharper and more well-defined edges in the reconstructed images, improving the visual detection of medical details and aiding clinicians in identifying anatomical structures and abnormalities more accurately.\\

In order to exploit a-priori information, in Section~\ref{sec:num_results} we also consider the Generalized $\ell_1 $-Total Variation (GenTV) regularizer, which ensures that the solution image closely resembles a possibly good guess of $\xHsim$,  in terms of gradient transforms. Denoting that guess as $\overline{\xH}$, GenTV reads:
\begin{equation}\label{eq:GenTV}
    \mathcal{R}(\bx) = \|\nabla \left( \bx - \overline{\xH} \right)\|_1.
\end{equation}
In our scenario, GenTV can easily take advantage of the end-to-end networks, and $\overline{\xH}$ can be computed as $\lowhigh(\xL)$. \\
We remark that our selection of the regularization function is merely illustrative, since the  LIP-CAR approach lets us accommodate models other than  \eqref{eq:LIPCAR} for regularizing the inverse problem. 

At last, we highlight that the minimization problem \eqref{eq:LIPCAR} must be solved with an iterative procedure to be set according to the mathematical properties of the regularization function.

\subsection{Convergence analysis of LIP-CAR}\label{ssec:conv}


We now prove that our LIP-CAR model \eqref{eq:LIPCAR} is well-posed and convergent; that is, it admits solutions that continuously depend on the data and converge to the true high-dose image $\xH$ in the limit of infinite NN parameters and zero noise. In particular, this section explicitly shows the dependence of the LIP forward operator $\highlow$ on the set $\Theta$ of trainable parameters introduced in Section \ref{ssec:NN}. To make our analysis more general, we do not prescribe a fixed $\mathcal{R}$.

We enumerate the final output of each training phase by $\Theta_{t,k}$, indicating that $\Theta_{t,k}$ is the set of free parameters fixed at the end of the $k$-th training, $k\geq 1$, as briefly discussed in Section \ref{ssec:NN}. Each training can differ from the others in many ways, such as the number of input-output pairs, the optimizer, and the training hyperparameters, to name a few. We do not discuss these differences in-depth. Let us just note that, even if the overall architecture of the NN is fixed, the number of parameters is allowed to increase, that is, $ \# \Theta_{t,k} \leq \# \Theta_{t,k+1}$. This increase reflects the idea that the depth of the NN (i.e. the number of layers) can be enhanced to achieve better accuracy.

For each $\Theta_{t,k}$, we have the corresponding trained LIP forward operator:
$$
\highlowk \colon \mathcal{X} \to \R^n_+.
$$
We need two batches of hypotheses. The first batch is as follows:

\begin{enumerate}[label=($\mathcal{H}$\arabic*)]
\item\label{h1} The set of (medical) images $\mathcal{X}$ is closed and regular; 
\item\label{h2} $\mathcal{F} \colon \mathcal{X} \to \mathcal{X}$ is bijective and continuous;
\item\label{h3} $\sup_{\bx\in \mathcal{X}}\|\highlowk(\bx) - \mathcal{F}(\bx)\|_\infty\leq M_k\to 0$ as $k\to \infty$.
\end{enumerate}

These hypotheses are very plausible. Hypothesis \ref{h1} is technical, and it serves only to avoid unrealistic pathological cases in real-world applications. Regularity means that the boundary of $\mathcal{X}$ is a set of zero Lebesgue measure, and it is safe to assume that every accumulation point of $\mathcal{X}$ is still a medical image.  Regarding \ref{h2}, the unknown operator $\mathcal{F}$ is assumed to have an inverse, or a pseudo-inverse, $\mathcal{F}^\dagger$, which represents the anti-diffusion of the CA concentration and maps low-doses images to high-dose images. Moreover, we can assume that each of the $n$-th components of $\mathcal{F}=(\mathcal{F}_1,\ldots,\mathcal{F}_n)$ are Lebesgue-measurable functions. Therefore, by Lusin-type theorems, it is possible to identify $\mathcal{F}$  with a continuous function outside a set of arbitrarily small Lebesgue measures.

Regarding the third hypothesis \ref{h3}, we recall that NNs possess universal approximation properties; see \cite{hornik1989multilayer} and especially \cite[Theorem 1]{zhou2020universality}. Thus, it is reasonable to assume that the sequence $\{\highlowk\}_{k\in\N}$ converges uniformly to $\mathcal{F}$, provided the right classes of architectures are chosen, and the training phases are sufficiently accurate.

We are in the position to begin our analysis.
\begin{definition}[Solutions]\label{def:solution}
Given $\xL \in \mathcal{X}$, we call $\xH$ the unique solution of:
\begin{equation*}\label{eq:LIP2}
\mathcal{F}(\bx) = \xL.
\end{equation*}
An element $\bx_k^\dagger \in \mathcal{X}$ is an $\mathcal{R}$-minimizing solution of the LIP problem~\eqref{eq:LIP}, in the limit case $\delta=0$, if $\highlowk(\bx_k^\dagger)=  \xL$ and:
\begin{equation*}
\mathcal{R}(\bx_k^\dagger) = \min \left\{\mathcal{R}(\bx) \mid \bx \in \mathcal{X},\; \highlowk(\bx)=  \xL \right\}.
\end{equation*}

\end{definition}

We introduce now the second batch of hypotheses:
\begin{enumerate}[label=($\mathcal{H}$\arabic*), resume]   
     \item\label{h4} $\highlowk$ is continuous for every $k$;
     \item\label{h5} $\highlowk(\mathcal{X})\supseteq \mathcal{F}(\mathcal{X})$;
    \item\label{h6} $\mathcal{R} \colon \mathcal{X} \to [0,+\infty)$ is  continuous.
\end{enumerate}
Notice that Hypothesis \ref{h4} is verified for every NN architecture used in applications since all the activation functions in the inference phase are typically component-wise Lipschitz continuous. Hypothesis \ref{h5} is instead a natural consequence of \ref{h3}.

In the next results, we provide well-posedness, stability, and convergence properties. The proofs are mostly standard, but for the convenience of the reader we highlight the key-points. We implicitly assume the validity of all the hypotheses \ref{h1}-\ref{h6} so far introduced.

Now, we define:
\begin{equation*}
\begin{aligned}
    &\Gamma_{k,i} (\bx) \coloneqq \frac{1}{2} \|\highlowk(\bx) - \xL^{\delta_{k,i}} \|_2^2 + \alpha_{k,i} \mathcal{R}(\bx), \\
    &\Gamma_{k} (\bx) \coloneqq \frac{1}{2} \|\highlowk(\bx) - \xL^{\delta_{k}} \|_2^2 + \alpha_{k} \mathcal{R}(\bx).
\end{aligned}
\end{equation*}
\begin{lemma}[Well-posedness]\label{lem:well-posedness}
    For every $k\in \N$ there exists an $\mathcal{R}$-minimizing solution $\bx_k^\dagger$, and the sets $\underset{\bx \in \mathcal{X}}{\argmin}\{\Gamma_{k,i}(\bx)\}$ and $\underset{\bx \in \mathcal{X}}{\argmin}\{\Gamma_{k}(\bx)\}$  are non-empty.
\end{lemma}
\begin{proof}

The subset $S \coloneqq \left\{\bx \in \mathcal{X} \mid \highlowk(\bx) = \xL \right\}$
is non-empty and compact. Thus, $\mathcal{R}_{|S}$ admits a minimizer.

In the same way, by continuity and compactness, $\Gamma_{k,i}$ and $\Gamma_{k}$ achieve a minimum on  $\mathcal{X}$.
\end{proof}

\begin{theorem}[Stability]\label{thm:stability}
Set a sequence $\delta_{k,i}$ of  noise intensities such that $\lim_i \delta_{k,i} = \delta_k >0$ for every $k$, and fix $\alpha_{k,i}=\alpha_k>0$ for every $i$. \\
Then any sequence $\{\bx_{\textnormal{H},k,i}^{\textnormal{sim}}\}_{i\in\N}$ of elements $\bx_{\textnormal{H},k,i}^{\textnormal{sim}}\in \underset{\bx \in \mathcal{X}}{\argmin}\{\Gamma_{k,i}(\bx)\}$ has a convergent subsequence such that:
$$\lim_{i'}\bx_{\textnormal{H},k,i'}^{\textnormal{sim}}= \xHsimk \in \underset{\bx \in \mathcal{X}}{\argmin}\{\Gamma_{k}(\bx)\}.$$
If  $\xHsimk$ is unique, then the whole sequence 
$\bx_{\textnormal{H},k,i}^{\textnormal{sim}}$ converges to $\xHsimk$.
\end{theorem}
\begin{proof}
By Lemma \ref{lem:well-posedness}, the sequence $\{\bx_{\textnormal{H},k,i}^{\textnormal{sim}}\}_{i\in\N}$ is well-posed. By compactness there exists a convergent subsequence $\lim_{i'}\bx_{\textnormal{H},k,i'}^{\textnormal{sim}}= \xHsimk$. By continuity and the definition of $\bx_{\textnormal{H},k,i'}^{\textnormal{sim}}$,
$$
\Gamma_k(\xHsimk) = \lim_{i'}\Gamma_{k,i'}(\bx_{\textnormal{H},k,i'}^{\textnormal{sim}})\leq \lim_{i'}\Gamma_{k,i'}(\bx) = \Gamma_{k}(\bx)
$$
for any $\bx\in\mathcal{X}$. Thus, $\xHsimk \in \underset{\bx \in \mathcal{X}}{\argmin}\{\Gamma_{k}(\bx)\}$. If $\xHsimk$ is unique, then every subsequence of $\{\bx_{\textnormal{H},k,i}^{\textnormal{sim}}\}_{i\in\N}$ has a convergent sub-subsequence to $\xHsimk$. By a standard topological argument, $\lim_i\bx_{\textnormal{H},k,i}^{\textnormal{sim}}= \xHsimk$.
\end{proof}

\begin{theorem}[Convergence]\label{thm:convergence}
Let $\xL\in\mathcal{X}$ be fixed and $\xH$ be the unique solution of $\mathcal{F}(\bx)=\xL$. Let $\{\delta_{k,i}\}_{k,i\in\N}$ be any sequence of strictly positive noise intensities such that $\lim_i \delta_{k,i} =0$ for every $k$.  Chose $\alpha_{k,i} \colon (0,+\infty) \to (0,+\infty)$ such that $\lim_i\alpha_{k,i} = 0$ and  $\lim_i \delta^2_{k,i}/\alpha_{k,i} = 0$ for every $k$.

Then any sequence  of elements $\xHsimki\in \underset{\bx \in \mathcal{X}}{\argmin}\{\Gamma_{k,i}(\bx)\}$ admits a subsequence  such that $\lim_{i'}\bx_{\textnormal{H},k,i'}^{\textnormal{sim}}=\bx_k^\dagger $.  If $\bx_k^\dagger$ is unique, then $\lim_i\xHsimki~=~\bx_k^\dagger$. Moreover, if $\lim_k\delta_{k,k} = \lim_k\alpha_{k,k} = 0$, then $\lim_k\bx_{\textnormal{H},k,k}^{\textnormal{sim}} = \xH$.
\end{theorem}
\begin{proof}
The first part is a straightforward application of standard techniques. See for example the proof in \cite[Theorem 3.26]{scherzer2009variational}. 
For every $k$, fix now an $\mathcal{R}$-minimizing solution $\bx_k^\dagger$. Then,
\begin{equation*}
 \Gamma_{k,k}(\bx_{\textnormal{H},k,k}^{\textnormal{sim}}) \leq \Gamma_{k,k}(\bx_k^\dagger)\leq  \delta_{k,k}^2 + \alpha_{k,k}\mathcal{R}(\bx_k^\dagger).
\end{equation*}
By passing to a subsequence, if necessary, we have that $\lim_{k'}\bx_{\textnormal{H},k',k'}^{\textnormal{sim}} = \bx^{*}$ and $\lim_{k'}\bx_{k'}^\dagger = \bx^{**}$, with $\bx^{*}, \bx^{**}\in \mathcal{X}$. Therefore:
$$
\lim_{k'}\Gamma_{k',k'}(\bx_{\textnormal{H},k',k'}^{\textnormal{sim}}) =0,
$$
and it follows that:
$$
\lim_{k'} \|\Psi_{\textnormal{H2L},k'}(\bx_{\textnormal{H},k',k'}^{\textnormal{sim}}) - \xL^{\delta_{k',k'}} \|_2 =0.
$$
Using now  the uniform convergence \ref{h3}, we get: 
$$
\xL=\lim_{k'}\Psi_{\textnormal{H2L},k'}(\bx_{\textnormal{H},k',k'}^{\textnormal{sim}}) = \mathcal{F}(\bx^*), 
$$
that is, $\bx^* = \xH$. Since $\xH$ is unique, by a standard topological argument, $\lim_k \bx_{\textnormal{H},k,k}^{\textnormal{sim}} = \xH$.
\end{proof}

\subsection{Robustness and stability of LIP-CAR}\label{ssec:robustness}
While powerful tools, NNs are typically susceptible to errors and vulnerabilities.
They learn to mimic the behavior of an operator through a statistical optimization process over a finite training set, but they must demonstrate efficiency and accuracy when presented with test data and unseen input. This property is typically referred to as ``robustness''.
It ensures that the network has effectively learned existing patterns rather than merely memorizing the training data. Therefore, evaluating our networks' performance on unseen data is crucial for assessing the true effectiveness of our proposed approach and its reliability for practical applications in CA imaging.

More precisely, referring to the wide literature about NN robustness, the ``generalization error'' for a trained network is typically defined as the difference between the performance of the network on the training data and its performance on unseen data from the same distribution. It quantifies how well the network can generalize its learned patterns to new, consistent examples (composing the so-called ``in-domain'' test set).
Experiments reported in Section~\ref{sec:num_results} will analyze the generalization ability of all the considered learned operators, by comparing their performance of training and testing data.

In addition, since the seminal paper \cite{goodfellow2014explaining} by Goodfellow et al., adversarial examples, crafted to deceive neural networks by making imperceptible or small changes to input data, have been used to highlight the vulnerability of neural networks on slightly inconsistent data. 
As anticipated in Section~\ref{sec:intro}, the presence of noise is unavoidable in medical applications, but intrinsic components as in \eqref{eq:noiseIn} are assumed to share similar properties over the entire data set, built under the same medical protocol. However, the presence of extra noise in a few patients' images represents a particularly significant and plausible perturbation that must be considered. An elevated noise level may result from many factors, including calibration issues or patient anomalies. In these cases, the unexpected perturbation is called ``out-of-domain'' noise, as it represents a statistically significant mismatch between training and test samples that can cause severe artifacts or even hallucinations on some output images.\\ 
Recalling numerical analysis classical theory, from now on, we call ``stability'' the robustness of an NN-based approach with respect to out-of-domain noise. 

\begin{remark}\label{rem:stability}
Theorem~\ref{thm:stability} in Section~\ref{ssec:conv} has already theoretically stated that, for every \(k\), the model~\eqref{eq:LIPCAR} is continuous and stable against extra perturbations \(\delta_{k,i} = \delta_k \pm \delta_i \neq \delta_k\). In Section \ref{ssec:results_IPRobustness}, we will assess the stability of the LIP-CAR approach numerically. 
These analyses are important when handling NNs since their performances can be heavily affected by generalization errors stemming from unseen data and types of perturbations different from the ones present in the training set. Assuming moreover that \(\mathcal{R}\) is convex, it is possible to control the rate of convergence of \(\bx_{\textnormal{H},k,i}^{\textnormal{sim}}\) towards \(\bx_k^\dagger\) using the Bregman distance \(D^\mathcal{R}_{\xi}(\cdot, \cdot)\), induced by \(\mathcal{R}\) in a fixed direction~\(\xi\), such that
$$
D^\mathcal{R}_{\xi}(\bx_{\textnormal{H},k,i}^{\textnormal{sim}},\bx_k^\dagger) \approx \delta_{k,i}.
$$
The proof is technical, so we omit it here. Interested readers can refer to \cite[Theorem 3.42]{scherzer2009variational} for details. The key point is that the regularization term \(\mathcal{R}\) controls and smooths out the perturbations in \(\bx_{\textnormal{H},k,i}^{\textnormal{sim}}\) caused by the extra noise \(\delta_{k,i}\). This presents another notable difference compared to the direct end-to-end method~\eqref{eq:EndToEnd}.
\end{remark}

%% file: SECTION_experimental_setup.tex
In this section we illustrate our experimental design, by describing the CA data set and our methodological and implementation choices, before analysing the numerical results in Section \ref{sec:num_results}.

\subsection{Data set}\label{ssec:setup_dataset}
We conduct our study on real contrast agent images from a pre-clinical trial. 
Procedures were conducted according to the national and international laws on experimental animal research (L.D. 26/2014; Directive 2010/63/EU) and under a specific Italian Ministerial Authorization (project research number 215/2020-PR), by CRB/Test Facility of Bracco Imaging SpA.
The original dataset is composed of 61 cranial MRI examinations of lab rats affected by an orthotopically induced C6 glioma. 
Specifically, each lab rat underwent an  MRI session composed of  three T1-weighted spin echo sequences (TR = 360 ms, TE = 5.51 ms, echo spacing = 5.514 ms, rare factor = 2, FOV = 32 mm x 32 mm, matrix size = 256 x 256, number of slice = 24, slice thickness 0.75 mm, number of averages = 4) respectively acquired before any administration (pre- CA), after administration of  0.01 mmol Gd/kg of a non-commercial high relaxivity dimeric gadolinium based CA (low-dose image), and after administration of additional 0.04 mmol Gd/kg of the same agent (a full-dose image).\\
Having MRI 3D data as stacks of 24 slices, we split the volumes into 2D images of 256x256 pixels. After cleaning the data set by removing slices with missing or poorly registered sequences, we randomly split the data into training and testing subsets.
We consider N = 840 images for training and N = 240 images for testing.


The presence of unavoidable intrinsic noise $\boldsymbol{\eta}^{in}$ on all the images made us consider pre-processing the data with a denoising tool before proceeding with training procedures. 
However, our analysis of noise statistics indicated that the high-dose images exhibited small perturbations on average, and we found it was not worth risking the loss of details and contrast reduction on the high-dose images. On the contrary, the low-dose slices were more significantly impacted by noise and the benefits of their denoising are reported and discussed in Section~\ref{ssec:results_dnn}.
In these tests, we opted for the BM3D filter~\cite{dabov2007image}, which is a block-matching and filtering method for denoising, and it represents a state-of-the-art tool that has been shown to be very effective in removing noise from images.

\subsection{Approaches and notations}\label{ssec:setup_notat}
As mentioned, the proposed LIP-CAR approach has a flexible design. 
For instance, we have considered both the LIP forward operators introduced in section \ref{ssec:NN}, i.e. $\lowhigh$ defined in Equation \eqref{eq:CAROperator} and $\lowhighp$ defined in Equation \eqref{eq:CAROperator_withP}. 
We have also exploited the modularity of the optimization problem statement and considered both the TV prior as in Equation \eqref{eq:TV} and the GenTV regularizer defined in Equation \eqref{eq:GenTV}.

We focus on the following three main learned inverse problems in the numerical experiments.\\
We denote as {\bf LIP-H2L-TV} and {\bf LIP-PH2L-TV} the cases where the Total Variation regularizer is considered and the forward operator is played by the $\highlow$ or $\highlowp$ networks, respectively. 
We remark that the  LIP-H2L-TV solver does not consider pre-dose images at all, hence it may be of interest for those cases where registered pre-dose and low-dose images are not available.\\
Then, the label  {\bf LIP-PH2L-GenTV} refers to the inverse problem solver where the forward operator is played by $\highlowp$ and the Generalized Total Variation regularizer is considered. In this case, GenTV relies on the $\overline{\xH}$ image computed by the $\lowhighp$ operator.
We observe that this implementation heavily exploits all the provided items and tools at best, by using the pre-dose images for training the LIP forward operator and using the images by the end-to-end approach inside the regularization term. 

Both the LIP forward operators and LIP-CAR framework are compared to the state-of-the-art, i.e. the trained networks $\lowhigh$ and $\lowhighp$. In the next section, where we compare all the computed $\xHsim$ images, we will use the notations {\bf NN-L2H} and {\bf NN-PL2H} to emphasize their reliance on the NN direct application \eqref{eq:EndToEnd}, in contrast to the LIP formulation.

\subsection{Metrics for quality assessment}\label{ssec:setup_metrics}
To quantify the performance of the considered approaches, we consider three quality assessment metrics:

\begin{enumerate}
\item Relative Error:
    \begin{equation*}
        \operatorname{RE}(\bx_1, \bx_2) \coloneqq \frac{\| \bx_1 - \bx_2 \|_2}{\|\bx_2\|_2}.
    \end{equation*}
    As the RE computes the $\ell_2$ distance between two images, the lower the RE, the better the performance of the NN.
\item Structural Similarity Index Measure (SSIM):  
    The $\operatorname{SSIM}(\bx_1, \bx_2)$ has been defined in \cite{wang2004image} to quantify the similarity of the $\bx_1$ image with respect to the reference $\bx_2$ image, capturing visual elements. The closer the SSIM is to 1, the better the performance of the NN.
\item Peak Signal-to-Noise Ratio:
    \begin{equation*}
         \operatorname{PSNR}(\bx_1, \bx_2) \coloneqq 20\log_{10}\left(\frac{\max_{i} | \bx_2(i) |}{\|\bx_1 - \bx_2 \|_2}  \right).
    \end{equation*}
    The higher the PSNR, the better the performance of the~NN.
\end{enumerate}
In addition to the assessment metrics, we will present images to examine the differences among methods and visually compare the contrast enhancements introduced onto the images.

\subsection{Network architecture and training}\label{ssec:setup_Architecture}

To lead fair comparisons, we always use the deep architecture considered in \cite{bone2021contrast} and in \cite{ammari2022can}, thus representing the state-of-the-art for the CAR imaging task.
In these manuscripts, the authors address the low-dose contrast agent problem with a variant of the U-net architecture, firstly introduced in \cite{ronneberger2015u} for bio-medical image segmentation, commonly referred to as V-Net. 
The reference architecture is thus a multiscale convolutional neural network, with strided convolutions instead of maxpooling layers for downsampling.
Our V-Net architecture has few differences from the original one, due to the different structure of the considered dataset. First of all, we consider an input layer with just low-dose/high-dose images, stacked with pre-dose images when available, differently to the original manuscript implementation \cite{bone2021contrast}, where the inputs were the concatenation of pre-dose, low-dose, T2-Flair, and diffusion images. Secondly, we modified the final sigmoid activation function with a ReLU to address stability issues in the training phase. 

When specified, the input low-dose images are substituted with their corresponding BM3D-filtered version.

In all our experiments, the NNs are trained for 80 epochs and batch size of $16$, with the Adam optimization algorithm \cite{kingma2014adam} and fixed step size of $10^{-3}$. They are trained to minimize the mean of the SSIM-based loss function $\operatorname{SSIM}(\cdot,\cdot)$ over the set of trainable parameters $\Theta$:
\begin{equation}\label{eq:loss}
    \operatorname{argmin}_\Theta \frac{1}{N}\sum_{i=1}^N \left(1 - \operatorname{SSIM}(\Psi(\bx^i_1),\bx^i_2)\right),
\end{equation}
where $\Psi$ is a generic NN with trainable parameters set $\Theta$, $N$ is the number of training data, $\bx^i_1$ are the inputs of the network and $\bx^i_2$ the target images.
 
All the training procedures were executed on an NVIDIA Quadro P6000 GPU card, requiring approximately 20 minutes for the whole training process.

\subsection{Specifications about the regularized model}\label{ssec:setup_regul}

To solve the regularized inverse problem, we have exploited the widely used Adam solver. 
We have always used the low-dose image as a starting iterate and run 150 iterations.

The regularization parameters $\alpha$ used in \eqref{eq:LIPCAR} have been chosen heuristically to better balance the visual appearance of the output and their metrics.
For the LIP-H2L-TV model, we have used $\alpha = 6e-3$, for the LIP-PH2L-TV  $\alpha = 3e-3$ while for the LIP-PH2L-GenTV $\alpha = 6e-3$.

To analyze the stability of the LIP-CAR approach in case of out-of-domain noise, we will alter the low-dose input image of the test set by adding a random perturbation $\boldsymbol{\eta}^{out}$ drawn from a Gaussian white distribution with fixed noise level. The corrupted test data thus reads:
\begin{equation}
    \bx_L^\delta + \boldsymbol{\eta}^{out}.
\end{equation}
In section \ref{ssec:results_IPRobustness} we consider several noise realizations with increasing intensities, up to 0.066 standard deviation. 
Due to the high perturbation, the regularization parameter $\alpha$ has been changed based on a logarithmic-spaced distribution in the interval $[3e-3, 0.01]$ for LIP-PH2L-TV approach, $[6e-3, 0.01]$ for the LIP-H2L-TV and LIP-PH2L-GenTV models. The networks are not re-trained.


%% file: SECTION_numerical_experiments.tex

In this section, we first compare the trained networks used as image-to-image operators, focusing on their accuracy performances and on their generalization ability.
Then we evaluate the proposed LIP-CAR approach on the in-domain test set in Section~\ref{ssec:results_IP} and analyze its stability with respect to extra out-of-domain noise in Section~\ref{ssec:results_IPRobustness}.\\

\subsection{Ablation study of image-to-image approaches}\label{ssec:results_dnn}

Here, we evaluate the performance of the NNs used as direct predictors, and we test the usefulness of the BM3D applied to the low-dose images.
In Table \ref{tab:end-to-end}, we report the mean and the standard deviation of the PSNR and SSIM values computed over the entire test set.
Recalling the definitions provided in Section~\ref{ssec:setup_metrics}, we highlight that the target image of the first four rows is the high-dose image, hence we have used $\bx_1 =~\Psi(\xL)$ and $\bx_2 = \xH$ (with the suitably trained networks $\Psi$). Conversely, in the last rows, we have set  $\bx_1 = \Psi(\xH)$ and $\bx_2 = \xL$.

\begin{table}[ht]
\begin{center}
\begin{tabular}{llcc}
\toprule
NN   & preprocessing on $\xL^\delta$ & PSNR & SSIM  \\
\midrule
\multirow{ 2}{*}{$\lowhigh$} & \textbf{None} &  \textbf{28.14 (0.833)}  & \textbf{0.855 (0.021)} \\
         &  BM3D & 28.79 (0.928)  & 0.859 (0.022)  \\
\cmidrule{2-4}
\multirow{ 2}{*}{$\lowhighp$}  & \textbf{None} & \textbf{28.89 (0.851)}  & \textbf{0.868 (0.022)}  \\
         & BM3D & 29.34 (0.907) & 0.866 (0.022)  \\
\midrule
\multirow{ 2}{*}{$\highlow$} & None & 31.92 (1.110) & 0.851 (0.025) \\
        & \textbf{BM3D} & \textbf{30.50 (0.939)}  & \textbf{0.854 (0.026)}  \\
\cmidrule{2-4}
\multirow{ 2}{*}{$\highlowp$} & None & 37.85 (0.885) & 0.919 (0.007)  \\
      & \textbf{BM3D} & \textbf{39.05 (1.073)} & \textbf{0.952 (0.007)}  \\
\bottomrule \vspace{1mm}
\end{tabular}
\caption{Performance comparison of the image-to-image operators. We report the metrics' mean and standard deviation (in brackets) on the full test set. We highlight in bold the operators used in the following experiments.}\label{tab:end-to-end}
\end{center}
\end{table}

Looking at Table~\ref{tab:end-to-end}, firstly, we observe that we always get high SSIM values, reflecting the capability of all the considered image-to-image operators to approach their targets. The very small standard deviation values further confirm this. \\
Secondly, we appreciate the benefits of the use of the pre-dose images, making the low-to-high approach increase of one percentage point in SSIM and the high-to-low one increase up to 10 points (0.952 of SSIM from 0.854) while strongly reducing the standard deviation (only 0.007 from 0.025). 
Using $\xP$ images also helps mitigate noise on the outputs, as the PSNR values increase in all the reported cases, up to 39.05 for the $\highlowp$ operator. \\ 
Thirdly, we highlight that learning the low-dose images produces the highest metrics, above all in terms of PSNR and independently on the use of the BM3D denoiser. This suggests that the LIP forward operators are simpler to learn than the end-to-end ones, as we had used very comparable training set-ups.\\
Lastly, we observe that using the BM3D to denoise the $\xL$ images has mainly improved the $\highlowp$ performance, whose PSNR gains more than one point and reaches its highest value. 
In the following, we will always consider the $\highlow$ and $\highlowp$  operators trained on denoised low-dose images, while we do not contemplate the BM3D anymore for $\lowhigh$ and $\lowhighp$. \\ 

\begin{figure}
    \centering
    \includegraphics[width=0.45\textwidth]{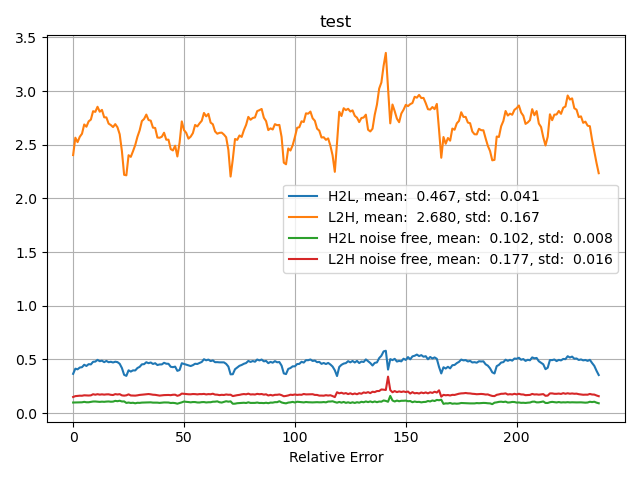}
    \includegraphics[width=0.45\textwidth]{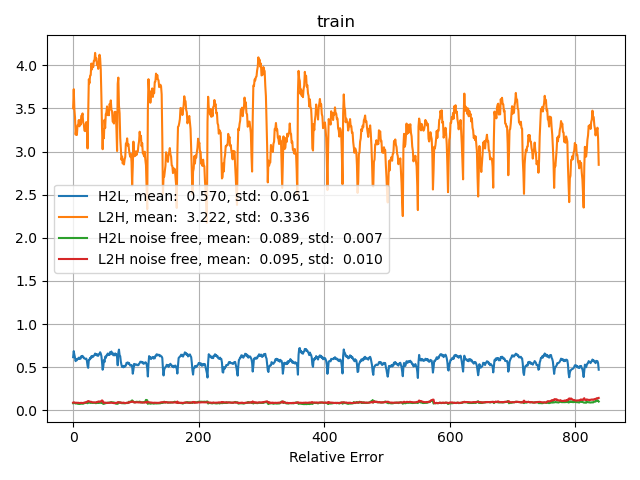}
\caption{Values of relative errors computed on each image of the test set (left) and train set (right), in case of in-domain input images and in case of out-of-domain noisy input images.}\label{fig:image-to-image}
\end{figure}

We need to check the robustness of the end-to-end and the LIP forward operators now.\\
As discussed in Section~\ref{ssec:robustness}, the network's capability to generalize from training data is an important factor to assess.
To estimate the generalization error due to the inherent differences between the statistical distribution of the training set (the only one considered for fine-tuning the learned parameters) and that of the test set, we carefully compare networks' performances on the two sets. 
In the top image of Figure~\ref{fig:image-to-image}, we plot the RE values across all the 240 test images, and in the bottom image, those across the 840 train images. 
Table \ref{tab:robustness_dnn} reports the corresponding mean and standard deviation values to quantify the different behaviors through averages.
We first focus on the in-domain cases, i.e., the actually acquired real data.
As visible, the overall trend does not change remarkably: the errors are always small, and those on the test data are only slightly higher than those on the training samples (above all for $\highlowp$ operator). In addition, there are extremely low variations from the means. We can conclude that the networks have not overfitted the training samples. \\
We also need to analyze the networks' stability with respect to additive out-of-domain noise $\boldsymbol{\eta}^{out}$, so we now apply the pre-trained $\highlowp$ and $\lowhighp$ operators onto input images corrupted by further Gaussian noise with standard deviation $\sigma=0.1$.
We have replicated the in-domain performance analysis on both sets, and the results are still in Figure~\ref{fig:image-to-image} and Table \ref{tab:robustness_dnn}.
As expected, the errors of the out-of-domain images are higher than the corresponding in-domain ones. However, if the $\highlowp$ approach (in red) still exhibits values around 0.5 with moderate fluctuations when applied on strongly noisy images, $\lowhighp$ (in green) presents higher values with wide oscillations. This empirically demonstrates the instability of the low-to-high approach and the superiority of the LIP forward operator for the considered added noise.

\begin{table}[ht]
\begin{center}
\begin{tabular}{llcc}
\toprule
    & NN  & in-domain & out-of-domain  \\
\midrule
\multirow{2}{*}{test set} & $\lowhighp$ & 0.177 (0.016) & 2.680 (0.167)  \\
    & $\highlowp$ & 0.102 (0.008) & 0.467 (0.041)  \\
 \midrule
\multirow{2}{*}{train set} & $\lowhighp$ & 0.095 (0.010) & 3.222 (0.336)  \\
 & $\highlowp$  & 0.089 (0.007) & 0.570 (0.061) \\
\bottomrule \vspace{1mm}
\end{tabular}
\caption{Mean and standard deviation (in brackets) of the relative errors computed on each image of the test and train sets, in case of in-domain input images and in case of out-of-domain noisy input images.}\label{tab:robustness_dnn}
\end{center}
\end{table}

\subsection{Accuracy of the LIP-CAR approach}\label{ssec:results_IP}

Now, we test the proposed LIP-CAR framework and compare its high-dose simulated images to the end-to-end generated ones. Here, we consider only the in-domain case, and we postpone the robustness analysis to Section~\ref{ssec:results_IPRobustness}.

\begin{figure*}
\centering
\begin{tabular}{ccccc}
    & $\bx_L$  & $\bx_H$  & NN-L2H  & LIP-H2L-TV \\
\rotatebox{90}{Image 39}&
\begin{tikzpicture}
    \node [anchor=south west, inner sep=0] (image) at (0,0) {\includegraphics[width=0.2\textwidth]{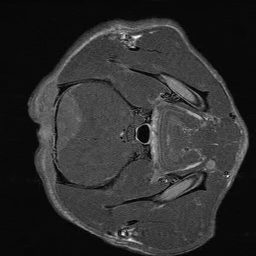}};
    \begin{scope}[x={(image.south east)}, y={(image.north west)}]
        \draw[red, thick] (0.25, 0.53) circle[radius=0.18];
    \end{scope}
\end{tikzpicture} &
\includegraphics[width=0.2\textwidth]{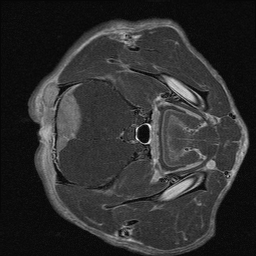} &
\includegraphics[width=0.2\textwidth]{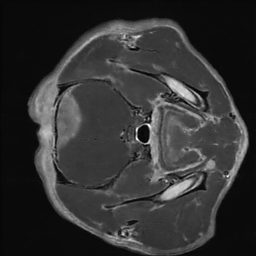} & 
\includegraphics[width=0.2\textwidth]{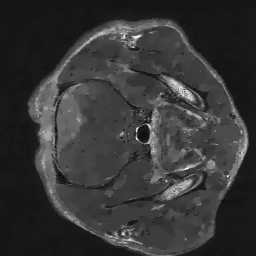} \\
\rotatebox{90}{Image 141} &
\begin{tikzpicture}
    \node [anchor=south west, inner sep=0] (image) at (0,0) {\includegraphics[width=0.2\textwidth]{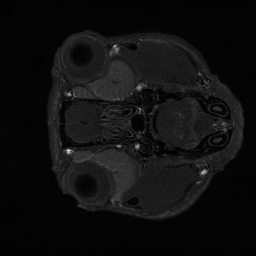}};
\end{tikzpicture} &
\includegraphics[width=0.2\textwidth]{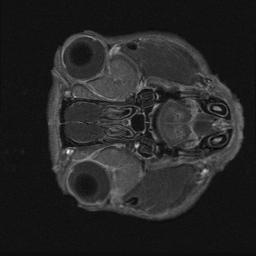}&
\includegraphics[width=0.2\textwidth]{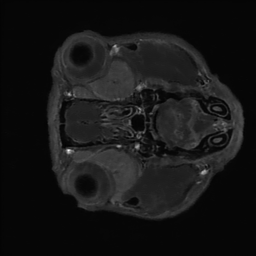}&
\includegraphics[width=0.2\textwidth]{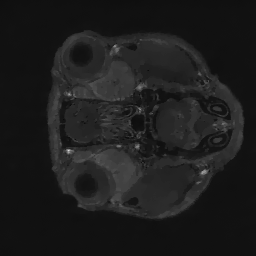}\\
\rotatebox{90}{Image 226} &
\includegraphics[width=0.2\textwidth]{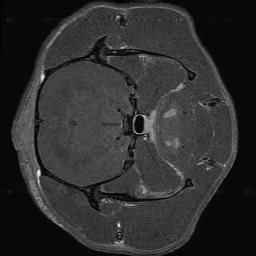}&
\includegraphics[width=0.2\textwidth]{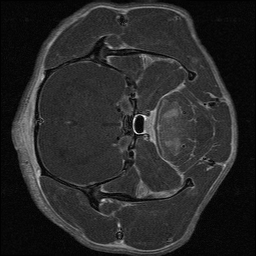}&
\includegraphics[width=0.2\textwidth]{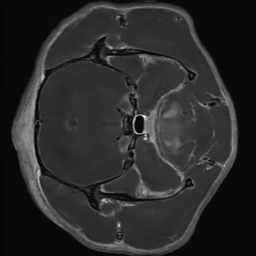}&
\includegraphics[width=0.2\textwidth]{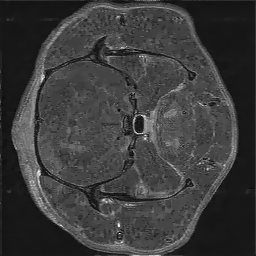}
\\
\end{tabular}
\caption{Results on test images number 39, 141 and 226, achieved from the low-dose $\bx_L$ images (first column) having the high-dose $\bx_H$ images (second column) as references: the NN-L2H framework based on the end-to-end  $\Phi_{L2H}$ operator (third column), the LIP-H2L-TV proposed method based on the forward $\Phi_{PH2L}$ operator and the TV prior \eqref{eq:TV} (last column).}
\label{fig:reco_L2H}
\end{figure*}
\begin{figure*}
\centering
\begin{tabular}{ccccc}
    & $\bx_P$  & NN-PL2H  & LIP-PH2L-TV  & LIP-PH2L-GenTV \\
\rotatebox{90}{Image 39} &
\includegraphics[width=0.2\textwidth]{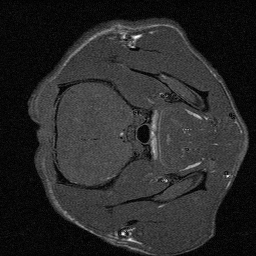}&
\includegraphics[width=0.2\textwidth]{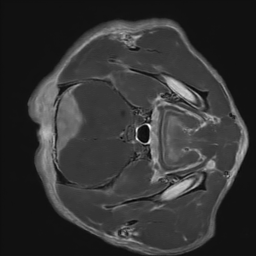}&
\includegraphics[width=0.2\textwidth]{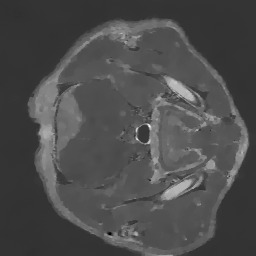}&
\includegraphics[width=0.2\textwidth]{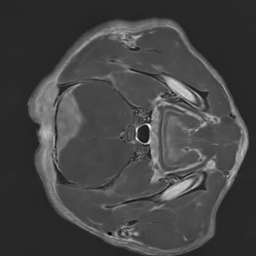}\\
\rotatebox{90}{Image 141} &
\includegraphics[width=0.2\textwidth]{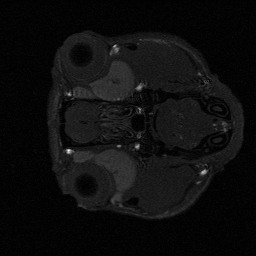}&
\includegraphics[width=0.2\textwidth]{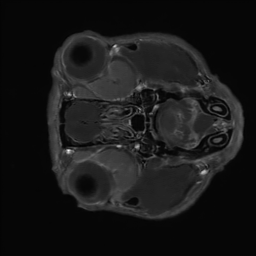}&
\includegraphics[width=0.2\textwidth]{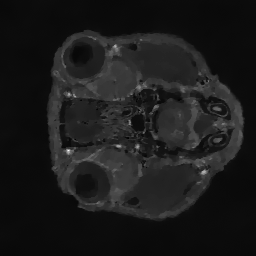}&
\includegraphics[width=0.2\textwidth]{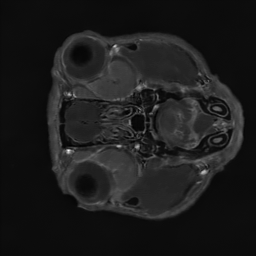}\\\
\rotatebox{90}{Image 226} &
\includegraphics[width=0.2\textwidth]{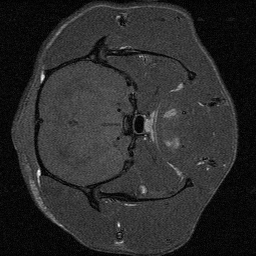}&
\includegraphics[width=0.2\textwidth]{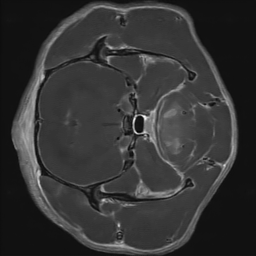}&
\includegraphics[width=0.2\textwidth]{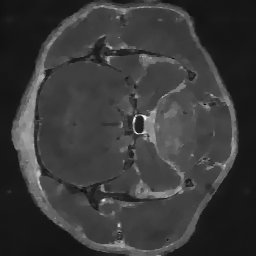}&
\includegraphics[width=0.2\textwidth]{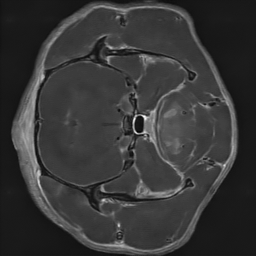}\\\
\end{tabular}
\caption{Results on test images number 39, 141 and 226, achieved when the pre-dose $\bx_P$ images (first column) are available: the NN-PL2H framework based on the end-to-end $\Phi_{PL2H}$ operator (second column), the LIP-PH2L-TV proposed method based on the forward $\Phi_{PH2L}$ operator and the TV prior \eqref{eq:TV} (third column), and the LIP-PH2L-GenTV proposed method based on the $\Phi_{PH2L}$ operator and the GenTV prior \eqref{eq:GenTV} (last column).}
\label{fig:reco_PL2H}
\end{figure*}

We start taking into consideration some reference test images in Figures \ref{fig:reco_L2H} and \ref{fig:reco_PL2H}.
Specifically, test image 39 depicts a slice with a large glioma; image 226 shares a similar shapes but does not contain a tumoral mass of interest, whereas image 141 shows a completely diverse head morphology including the olfactory bulb and the eyes of the mouse.

In Figure \ref{fig:reco_L2H}, we illustrate the basic scenario where only the $\xL$ and the $\xH$ data are available for training: in this case, only the NN-L2H and the LIP-H2L-TV methods are usable to simulate high-dose images.  
We immediately observe that both approaches are able to detect all the areas that need to be enhanced/de-enhanced, without adding any false enhancements/de-enhancements (potentially causing hallucinations and false diagnosis). However, the LIP-CAR solutions are slightly noisy in the background and darker than the images by NN-L2H. 

In Figure~\ref{fig:reco_PL2H}, we assume the pre-dose injection images $\xP$ are available, and we can add them as input data for the training. 
As visible, the pre-dose image 39 is very noisy and its tumoral mass is not evident, while it is perfectly enhanced with neat boundaries in all the simulated high-dose images.
In addition, all the computed images are well denoised, and they recovered fine details that were visible only in the corresponding high-dose slices (see, for instance, image 141).
This visual inspection suggests that the addition of the pre-dose images strongly improves the contrast among relevant anatomical areas and, again, no false positives are added.\\
Focusing on the regularizers used in the LIP approach (third and fourth columns in Figure \ref{fig:reco_PL2H}), we observe that the TV prior tends to smooth the reconstructions, while GenTV has remarkably taken advantage of the $\highlowp$ operator, as its images are bright well contrasted. 
However, our LIP approach overcomes NN-PL2H firmly, in terms of metrics. Table \ref{tab:MethodComparison} reports the SSIM values with respect to the high-dose target, for the pre-dose, the low-dose and all the simulated high-dose images. The values confirm that we always obtain reliable approximations of the reference high-dose images, as the starting SSIM values of the low-dose images increase remarkably, even with the naive LIP-H2L-TV approach. 
In some cases, the values of the NN approaches are slightly higher than the TV-based LIP approaches, but they do not achieve the highest metrics, which were always hit by LIP-PH2L-GenTV. 

\begin{table}[h]
	\begin{center}
		\begin{tabular}{lccc}
			\toprule
			Approach &  \multicolumn{3}{c}{SSIM values} \\
			& 39th image & 141th image & 226th image \\
			\midrule
			pre-dose        & 0.4118    & 0.3697      & 0.4512  \\
			low-dose        & 0.6627    & 0.6330     & 0.6821 \\
			NN-L2H          & 0.8646    & 0.8346     & 0.8680 \\
			NN-PL2H         & 0.8756    & 0.8488     & 0.8804 \\
			LIP-H2L-TV      & 0.8213    & 0.8376     & 0.8288 \\
			LIP-PH2L-TV     & 0.8713    & 0.8760     & 0.8690\\
			LIP-PH2L-GenTV  & 0.9007    & 0.8883     & 0.9020 \\
			\bottomrule \vspace{1mm}
		\end{tabular}\caption{SSIM values of the images reconstructed by the selected approaches, starting from the low-dose image, for three test images.}\label{tab:MethodComparison} 
	\end{center}
\end{table}

\begin{figure}
    \begin{subfigure}[t]{0.47\linewidth}
    \begin{tabular}{p{21mm} p{21mm} p{30mm}}
     \hspace{9mm}$\bx_H$  &  \hspace{3mm} NN-PL2H  & LIP-H2L-GenTV \\
    \begin{tikzpicture}
        \node [anchor=south west, inner sep=0] (image) at (0,0) {\includegraphics[width=0.3\textwidth]{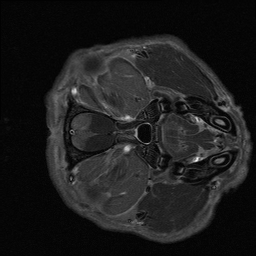}};
        \begin{scope}[x={(image.south east)}, y={(image.north west)}]
            \draw[red, thick] (0.32, 0.48) circle[radius=0.15];
        \end{scope}
    \end{tikzpicture} & 
    \includegraphics[width=0.3\textwidth]{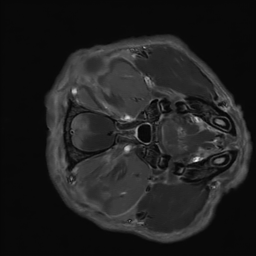}&
    \includegraphics[width=0.3\textwidth]{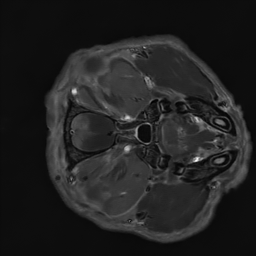}\\
    
    \begin{tikzpicture}
        \node [anchor=south west, inner sep=0] (image) at (0,0) {\includegraphics[width=0.3\textwidth]{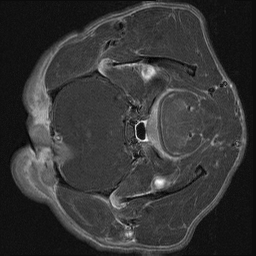}};
        \begin{scope}[x={(image.south east)}, y={(image.north west)}]
            \draw[red, thick] (0.24, 0.40) circle[radius=0.15];
        \end{scope}
    \end{tikzpicture} &
    \includegraphics[width=0.3\textwidth]{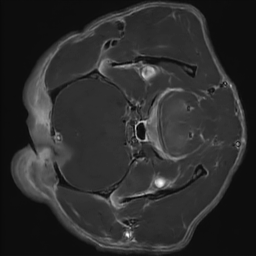}&
    \includegraphics[width=0.3\textwidth]{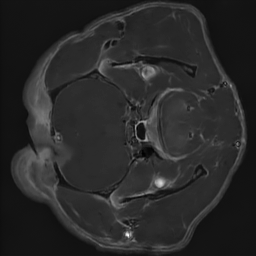}\\
    
    \begin{tikzpicture}
        \node [anchor=south west, inner sep=0] (image) at (0,0) {\includegraphics[width=0.3\textwidth]{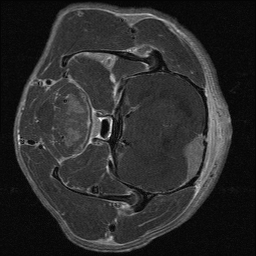}};
        \begin{scope}[x={(image.south east)}, y={(image.north west)}]
            \draw[red, thick] (0.75, 0.38) circle[radius=0.15];
        \end{scope}
    \end{tikzpicture} & 
    \includegraphics[width=0.3\textwidth]{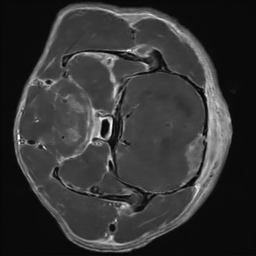}&
    \includegraphics[width=0.3\textwidth]{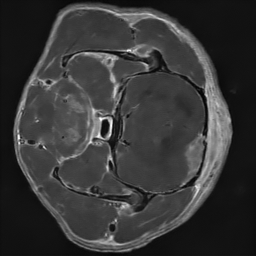}
    \end{tabular}
    \caption{}
    \end{subfigure}
 \hspace{3mm}
    \begin{subfigure}[t]{0.47\linewidth}
    \centering
    \begin{tabular}{p{21mm} p{21mm} p{30mm}} 
    \hspace{9mm} $\bx_H$  &  \hspace{3mm}NN-PL2H  & LIP-H2L-GenTV \\
    \includegraphics[width=0.3\textwidth]{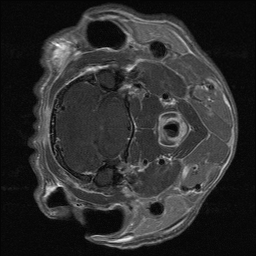}&
    \includegraphics[width=0.3\textwidth]{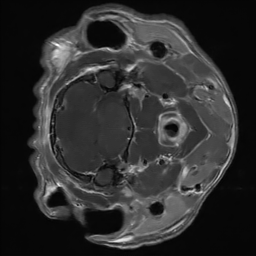}&
    \includegraphics[width=0.3\textwidth]{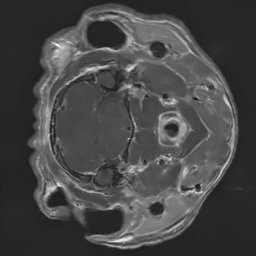}\\
    \includegraphics[width=0.3\textwidth]{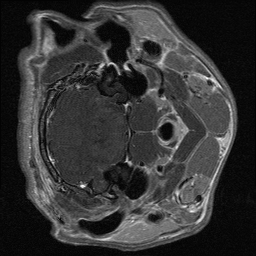}&
    \includegraphics[width=0.3\textwidth]{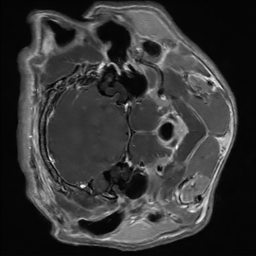}&
    \includegraphics[width=0.3\textwidth]{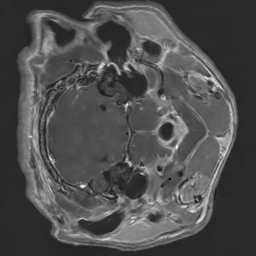}\\
    \includegraphics[width=0.3\textwidth]{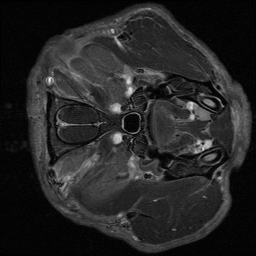}&
    \includegraphics[width=0.3\textwidth]{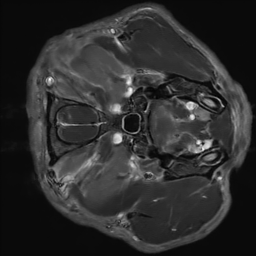}&
    \includegraphics[width=0.3\textwidth]{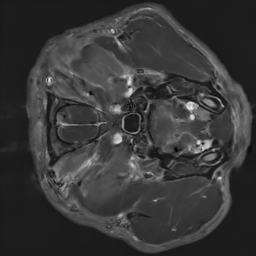}\\
    \end{tabular}
    \caption{}
    \end{subfigure}
    \caption{Results achieved on further test images, exploiting the pre-dose images. (a) Cases containing tumors, highlighted with red circles on the high-dose images; (b) cases containing no tumors.}
    \label{fig:testimages}
\end{figure}
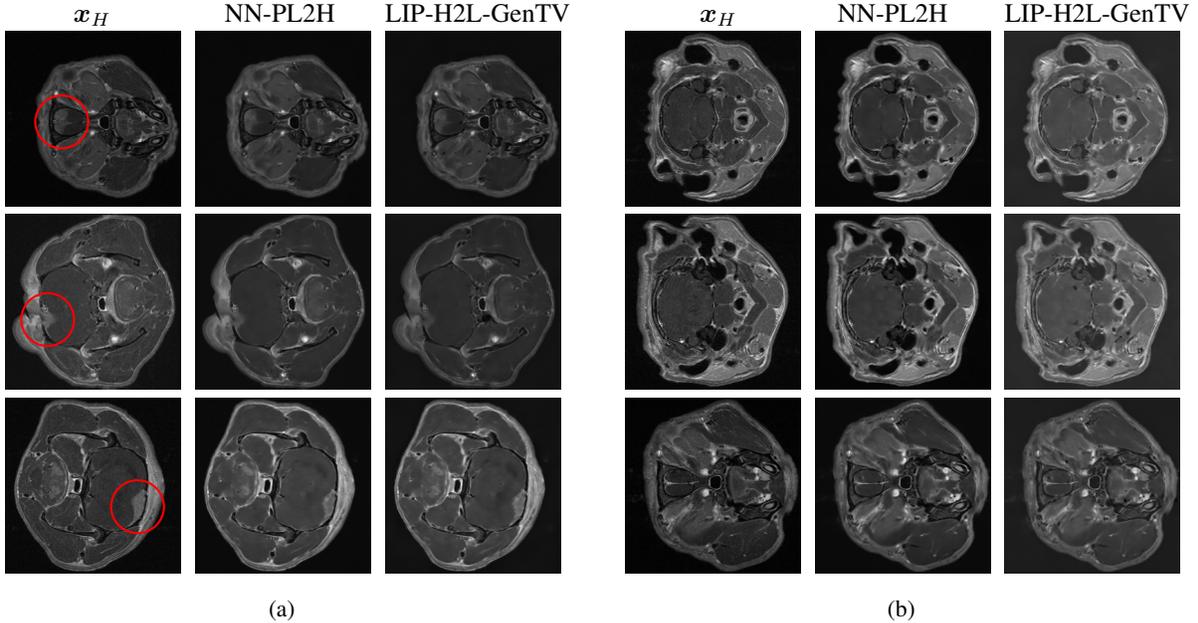

For a further visual inspection, in Figure  \ref{fig:testimages}, we show results on several test images depicting different elements of the mice heads. Here, we can further appreciate the capability of the NN-PL2H and LIP-PH2L-GenTV frameworks to always approximate the high-dose images carefully.
In these images, the contrast of tumoral masses has been increased (see red circles), and meanwhile, no false positives have been generated in slices with no tumors, neither artifacts nor hallucinations. 
We conclude by noting that the LIP-PH2L-GenTV images are similar to the NN-PL2H ones at visual inspection but have better metrics. Therefore, our mathematically-grounded approach is competitive with the state-of-the-art. 

\subsection{Stability of the LIP-CAR approach with respect to noise}\label{ssec:results_IPRobustness}
We now compare the methods on test images affected by additive $\boldsymbol{\eta}^{out}$ noise. We start considering the test image 39. 
We consider various noise intensities, given by standard deviation $\sigma$, and we report here the results for equispaced $\sigma$ values in the $[0, 0.066]$ range. 
In Figure \ref{fig:ssim_robustness}, we plot the behavior of the SSIM values with respect to $\sigma$, for three LIP solvers and the state-of-the-art methods.
When the additive noise is absent ($\sigma=0$) or very small ($\sigma=0.11$), the methods using the pre-dose images provide very high SSIM, strongly overcoming the LIP-H2L-TV method. The NN-L2H approach gives a medium-high quality solution as well.
In the case of higher perturbations (right-hand side of the plot), there is a trend reversal where the NN-L2H, NN-PL2H, and LIP-PH2L-GenTV produce increasingly poor high-dose simulations as sigma increases.
The GenTV constraint, which forces LIP-PH2L-GenTV ``to remain close'' to the  NN-PL2H image, is negatively influenced by the instability of the NN-PL2H approach, as already observed in Section~\ref{ssec:results_dnn}. The regularization of the inverse problem formulation only partially mitigates this effect, as LIP-PH2L-GenTV consistently maintains higher SSIM values than NN-PL2H.
Interestingly, the two TV-regularized approaches remain stable and reliable in high out-of-domain noise, always providing SSIM values higher than 0.7.

For a broader comparison between the in- and the out-of-domain cases, we consider the SSIM values computed on the whole test set for $\sigma=0$ and $\sigma=0.033$. The corresponding boxplots are shown in Figure \ref{fig:ssim_robustness_boxplot}.
The median values are mostly coherent with the values depicted in Figure~\ref{fig:ssim_robustness} about the 39th image, and it demonstrates that we did not cherry-pick that image.
We also observe that the amplitudes of the boxplots are always very small, in the case of $\sigma=0$, and a few outliers occurred. 
Conversely, when $\sigma=0.033$, only our TV-based LIP approaches show limited boxplots with short whiskers, while the NN-L2H and NN-PL2H approaches manifest lower SSIM medians and higher fluctuations among values. \\
This analysis confirms that the presence of out-of-domain noise does always affect the computations, but in different ways, and the LIP-CAR approach demonstrates a stable behavior making it a reliable tool when noisy data must be processed. 
Note that all the regularization terms we use are convex, and our stability comparison analysis accurately reflects the observations made in \ref{ssec:robustness} and in Remark \ref{rem:stability} specifically.

\begin{figure}
    \centering
    \begin{subfigure}[t]{0.45\textwidth}
        \centering
        \includegraphics[width=\textwidth]{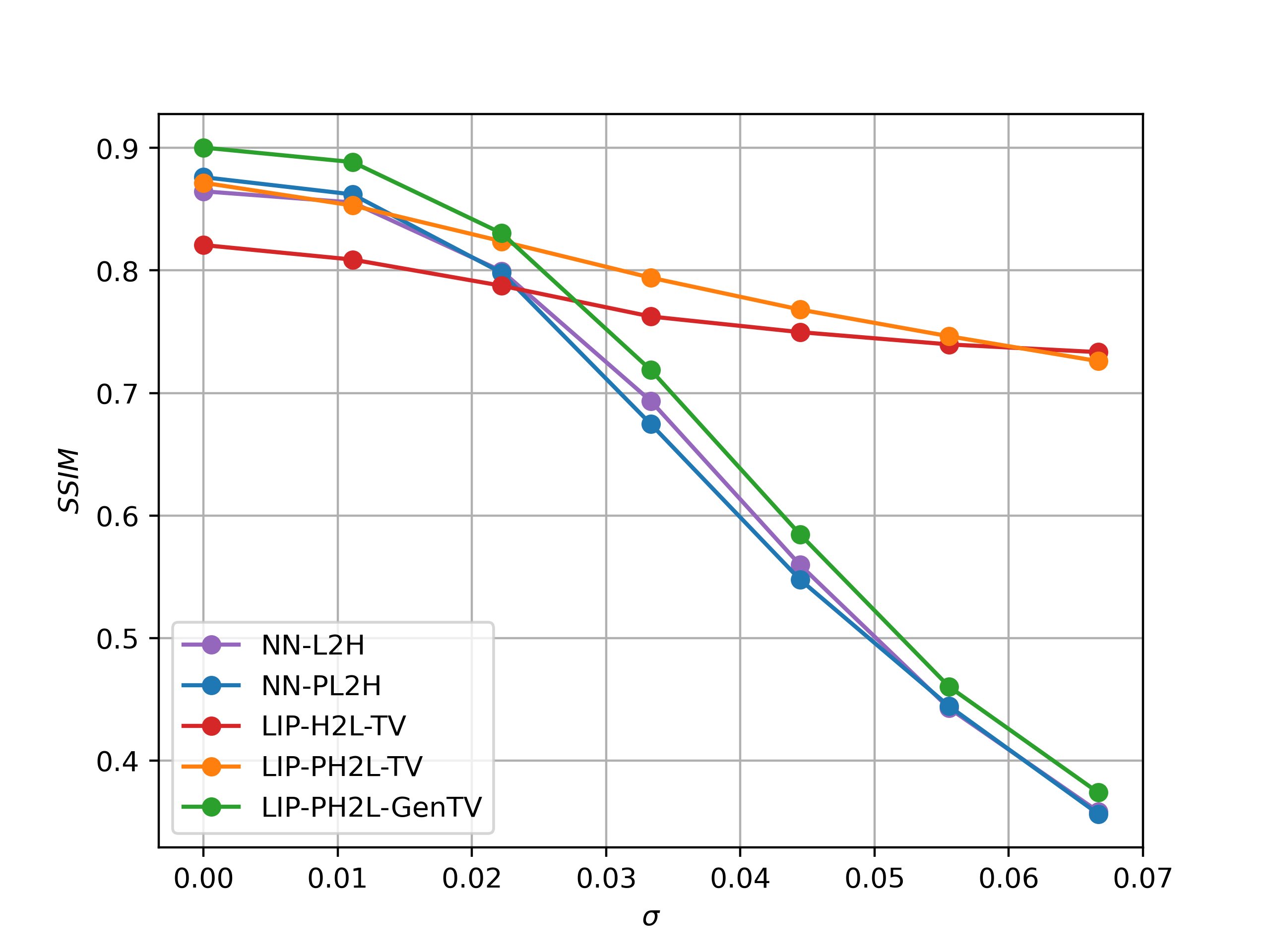}
        \caption{}
        \label{fig:ssim_robustness}
    \end{subfigure}
    \quad
    \begin{subfigure}[t]{0.45\textwidth}
        \centering
        \includegraphics[width=\textwidth]{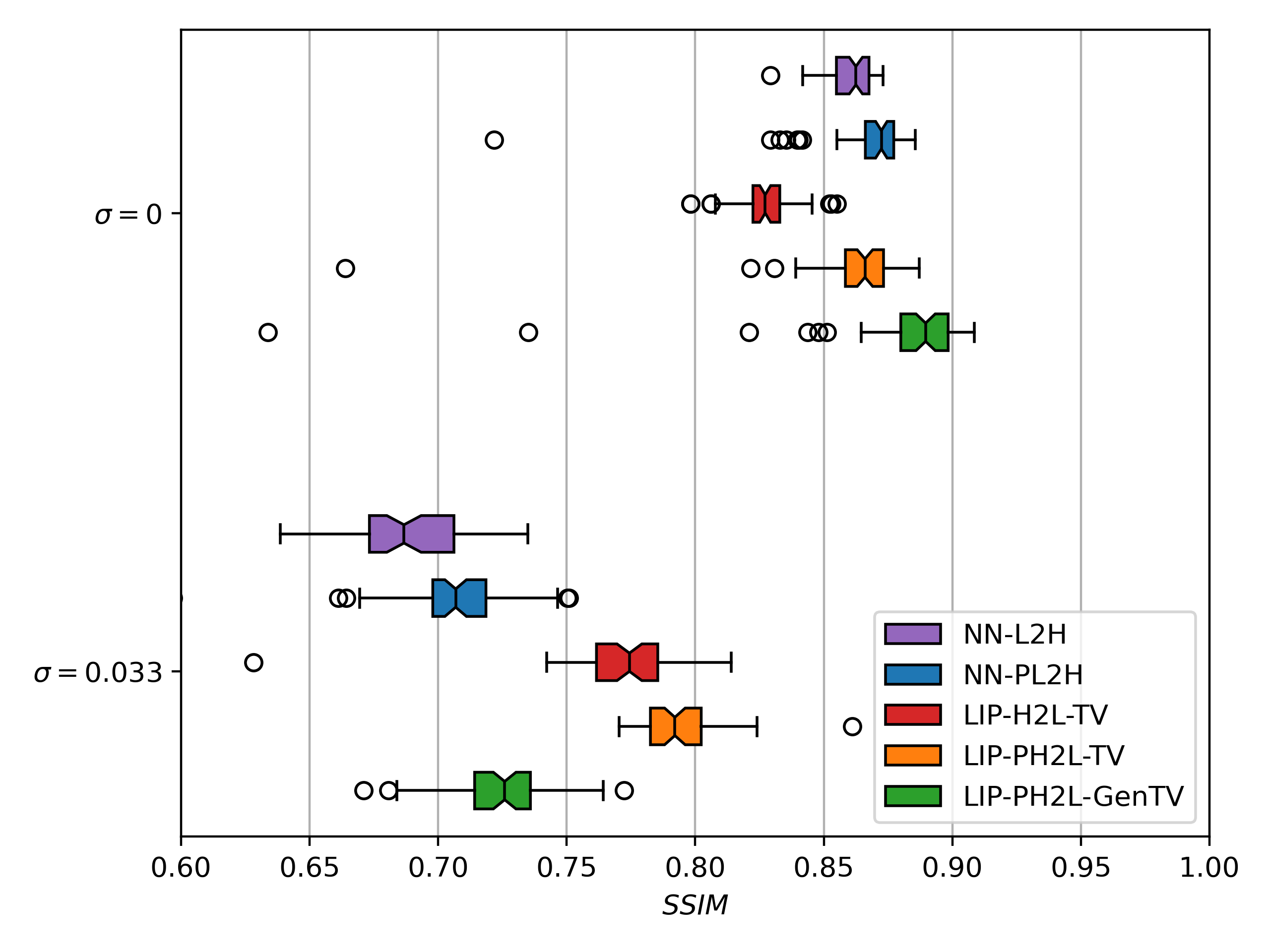}
        \caption{}
        \label{fig:ssim_robustness_boxplot}
    \end{subfigure}
    \caption{SSIM values on test images, when white Gaussian noise (with standard deviation $\sigma$) is added onto the low dose image. (a) SSIM at increasing $\sigma \in [0, 0.066]$ for image 39. (b) Boxplot comparing noiseless results with results obtained by adding noise with standard deviation $\sigma = 0.022$. }
\end{figure}

Finally, in Figure \ref{fig:visual_comparison}, we illustrate the NN-PL2H, LIP-PH2L-TV, and LIP-PH2L-GenTV images computed for the test image 39 at different noise intensities. In the first row, the out-of-domain noise component is null ($\sigma = 0$), and it serves as a reference. 
For each simulated high-dose image, the figure includes a cropped zoom-in over the tumoral mass.
It emerges clearly that while the computed images are quite similar in the in-domain test, the differences among the simulated high-dose outputs become more and more evident as $\sigma$ increases, confirming the different trends of the previous plots.
Within the NN-PL2H and the LIP-PH2L-GenTV frameworks, the images are more and more corrupted by a noisy pattern but small and thin details remain almost visible, whereas the TV prior smoothes fine details to preserve high contrasts and a high detachability of the main areas of interest.
Notably, the shape of the tumoral mass gets altered by the NN-PL2H processing for $\sigma > 0.022$ and almost vanishes for $\sigma = 0.044$.

\begin{figure*}
\centering
\begin{tabular}{l cc cc cc}
 & \multicolumn{2}{c}{NN-PL2H} & \multicolumn{2}{c}{LIP-PH2L-TV } & \multicolumn{2}{c}{LIP-PH2L-GenTV} \\
\rotatebox{90}{$\sigma = 0$} & 
    \begin{tikzpicture}
        \node [anchor=south west, inner sep=0] (image) at (0,0) {\includegraphics[height=0.15\textwidth]{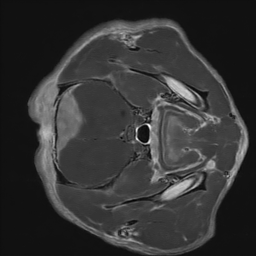}};
        \begin{scope}[x={(image.south east)}, y={(image.north west)}]
            \draw[red, thick] (0.15, 0.25) rectangle (0.48, 0.75);
        \end{scope}
    \end{tikzpicture} &
    \includegraphics[trim=10mm 17mm 35mm 17mm, clip, height=0.15\textwidth]{figures/stability/rec_PL2H_Unet_Bone_SSIM_39_sigma_0.000.png} \hspace{2mm} &
    \includegraphics[height=0.15\textwidth]{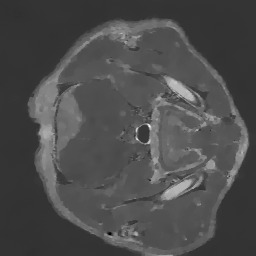} &
    \includegraphics[trim=10mm 17mm 35mm 17mm, clip, height=0.15\textwidth]{figures/stability/rec_PInvH2L_BM3D_Unet_Bone_SSIM_TV_39_sigma_0.000.png}  \hspace{2mm} &
    \includegraphics[height=0.15\textwidth]{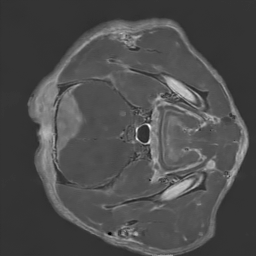} &
    \includegraphics[trim=10mm 17mm 35mm 17mm, clip, height=0.15\textwidth]{figures/stability/rec_PInvH2L_BM3D_Unet_Bone_SSIM_GenTV_39_sigma_0.000.png} \\    
\rotatebox{90}{$\sigma = 0.11$} & 
    \includegraphics[width=0.15\textwidth]{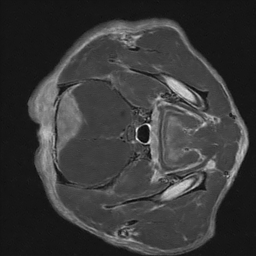} &
    \includegraphics[trim=10mm 17mm 35mm 17mm, clip, height=0.15\textwidth]{figures/stability/rec_PL2H_Unet_Bone_SSIM_39_sigma_0.011.png}  \hspace{2mm} &
    \includegraphics[width=0.15\textwidth]{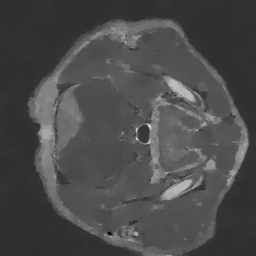} &
    \includegraphics[trim=10mm 17mm 35mm 17mm, clip, height=0.15\textwidth]{figures/stability/rec_PInvH2L_BM3D_Unet_Bone_SSIM_TV_39_sigma_0.011.png}  \hspace{2mm} &
    \includegraphics[width=0.15\textwidth]{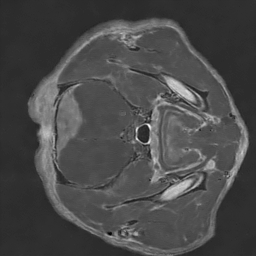} &
    \includegraphics[trim=10mm 17mm 35mm 17mm, clip, height=0.15\textwidth]{figures/stability/rec_PInvH2L_BM3D_Unet_Bone_SSIM_GenTV_39_sigma_0.011.png} \\  
\rotatebox{90}{$\sigma = 0.22$} & 
    \includegraphics[width=0.15\textwidth]{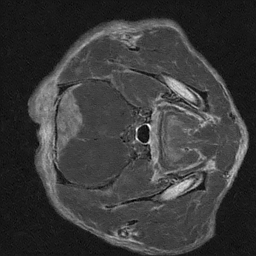} &
    \includegraphics[trim=10mm 17mm 35mm 17mm, clip, height=0.15\textwidth]{figures/stability/rec_PL2H_Unet_Bone_SSIM_39_sigma_0.022.png}  \hspace{2mm} &
    \includegraphics[width=0.15\textwidth]{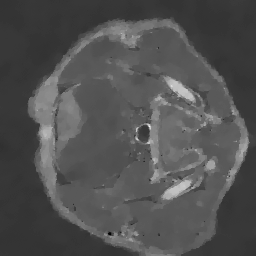} &
    \includegraphics[trim=10mm 17mm 35mm 17mm, clip, height=0.15\textwidth]{figures/stability/rec_PInvH2L_BM3D_Unet_Bone_SSIM_TV_39_sigma_0.022.png}  \hspace{2mm} &
    \includegraphics[width=0.15\textwidth]{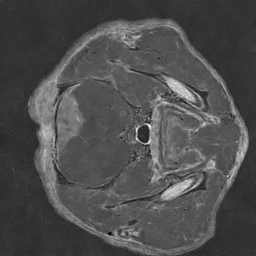} &
    \includegraphics[trim=10mm 17mm 35mm 17mm, clip, height=0.15\textwidth]{figures/stability/rec_PInvH2L_BM3D_Unet_Bone_SSIM_GenTV_39_sigma_0.022.png} \\  
\rotatebox{90}{$\sigma = 0.33$} & 
    \includegraphics[width=0.15\textwidth]{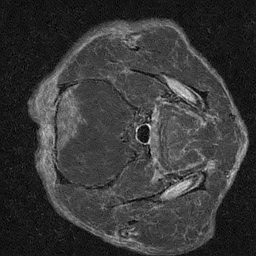} &
    \includegraphics[trim=10mm 17mm 35mm 17mm, clip, height=0.15\textwidth]{figures/stability/rec_PL2H_Unet_Bone_SSIM_39_sigma_0.033.png}  \hspace{2mm} &
    \includegraphics[width=0.15\textwidth]{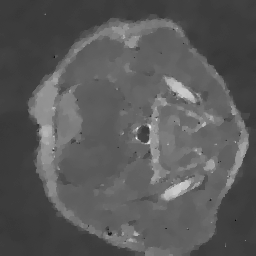} &
    \includegraphics[trim=10mm 17mm 35mm 17mm, clip, height=0.15\textwidth]{figures/stability/rec_PInvH2L_BM3D_Unet_Bone_SSIM_TV_39_sigma_0.033.png}  \hspace{2mm} &
    \includegraphics[width=0.15\textwidth]{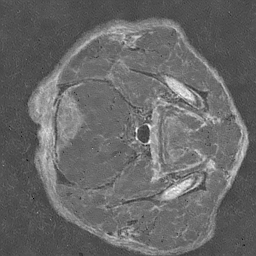} &
    \includegraphics[trim=10mm 17mm 35mm 17mm, clip, height=0.15\textwidth]{figures/stability/rec_PInvH2L_BM3D_Unet_Bone_SSIM_GenTV_39_sigma_0.033.png} \\  
\rotatebox{90}{$\sigma = 0.44$} & 
    \includegraphics[width=0.15\textwidth]{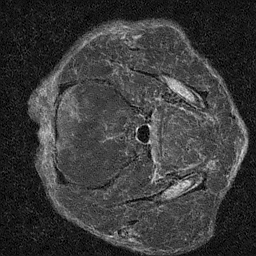} &
    \includegraphics[trim=10mm 17mm 35mm 17mm, clip, height=0.15\textwidth]{figures/stability/rec_PL2H_Unet_Bone_SSIM_39_sigma_0.044.png}  \hspace{2mm} &
    \includegraphics[width=0.15\textwidth]{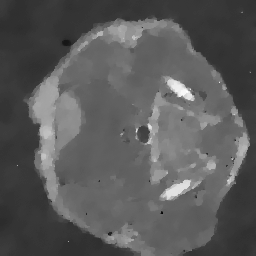} &
    \includegraphics[trim=10mm 17mm 35mm 17mm, clip, height=0.15\textwidth]{figures/stability/rec_PInvH2L_BM3D_Unet_Bone_SSIM_TV_39_sigma_0.044.png}  \hspace{2mm} &
    \includegraphics[width=0.15\textwidth]{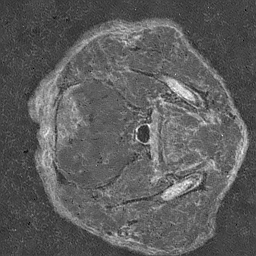} &
    \includegraphics[trim=10mm 17mm 35mm 17mm, clip, height=0.15\textwidth]{figures/stability/rec_PInvH2L_BM3D_Unet_Bone_SSIM_GenTV_39_sigma_0.044.png} \\ 

\end{tabular}
\caption{Reconstructed images by NN-PL2H, IP-PH2L-TV, and IP-PH2L-GenTV at varying noise intensities, and the corresponding magnifications on the area depicted by the red rectangle,  with the tumoral mass that must be enhanced with the contrast-agent digital simulation. }
\label{fig:visual_comparison}
\end{figure*}